%% file: arxiv.tex
\documentclass[10pt]{article}
\usepackage{graphicx} 
\usepackage[top=2.25cm, bottom=2.5cm, left=2.5cm, right=2.5cm, columnsep=0.65cm]{geometry}
\usepackage{amsfonts,amsmath,amssymb,amsthm}
\usepackage[OT1]{fontenc}
\usepackage{natbib}
\usepackage[usenames]{color}
\usepackage[most]{tcolorbox}
\usepackage{url}
\usepackage{mathrsfs}
\usepackage{cancel}
\usepackage{dsfont}
\usepackage{xcolor}
\usepackage{parskip}
\usepackage{natbib}
\usepackage{algorithm}
\usepackage{algorithmic}
\usepackage{enumitem}
\usepackage{multirow}
\usepackage{adjustbox}
\usepackage{makecell}
\usepackage{booktabs}
\usepackage{array}
\usepackage{subcaption}
\usepackage{wrapfig}
\usepackage{authblk}
\definecolor{redd}{HTML}{FF3B3F}
\definecolor{bluee}{HTML}{4169E1}
\usepackage[
    colorlinks,
    linkcolor=redd,
    anchorcolor=bluee,
    citecolor=magenta
]{hyperref}
\setitemize{noitemsep,topsep=0pt,parsep=0pt,partopsep=0pt}

\newtheorem{theorem}{Theorem}[section]
\newtheorem{proposition}{Proposition}[section]
\newtheorem{assumption}{Assumption}[section]
\newtheorem{definition}{Definition}[section]
\newtheorem{lemma}{Lemma}[section]
\newtheorem*{remark}{Remark}
\input{math_commands.tex}

\title{Local Linear Attention: An Optimal Interpolation of Linear and Softmax Attention For Test-Time Regression}
\date{}
\author[1]{Yifei Zuo\thanks{Work done at Snowflake AI Research. }\thanks{\texttt{\{yifeizuo2029,yutongyin2028\}@u.northwestern.edu,zhaoranwang@gmail.com}}}
\author[1]{Yutong Yin\textsuperscript{$\dagger$}}
\author[2]{Zhichen Zeng\thanks{\texttt{\{zczeng,angliz,banghua\}@uw.edu}}}
\author[2]{Ang Li\textsuperscript{$\ddagger$}}
\author[2]{Banghua Zhu\textsuperscript{$\ddagger$}}
\author[1]{Zhaoran Wang\textsuperscript{$\dagger$}}

\affil[1]{Northwestern University}
\affil[2]{University of Washington}

\begin{document}
\maketitle
\begin{abstract}
Transformer architectures have achieved remarkable success in various domains.
While efficient alternatives to Softmax Attention have been widely studied,
the search for more expressive mechanisms grounded in theoretical insight—even at greater computational cost—has been relatively underexplored.
In this work, we bridge this gap by proposing Local Linear Attention (LLA), a novel attention mechanism derived from nonparametric statistics through the lens of test-time regression.
First, we show that LLA offers theoretical advantages over Linear and Softmax Attention for associative memory via a bias-variance trade-off analysis.
Next, we address its computational challenges and propose two memory-efficient primitives to tackle the \(\Theta(n^2d)\) and \(\Theta(nd^2)\) complexity.
We then introduce {FlashLLA}, a hardware-efficient, blockwise algorithm that enables scalable and parallel computation on modern accelerators.
In addition, we implement and profile a customized inference kernel that significantly reduces memory overheads.
Finally, we empirically validate the advantages and limitations of LLA on test-time regression, in-context regression, associative recall and state tracking tasks.
Experiment results demonstrate that LLA effectively adapts to non-stationarity, outperforming strong baselines in test-time training and in-context learning, and exhibiting promising evidence for its scalability and applicability in large-scale models.
Code is available at \url{https://github.com/Yifei-Zuo/Flash-LLA}.
\end{abstract}
\input{sections/introduction.tex}

\input{sections/prelim.tex}
\input{sections/lla.tex}

\input{sections/experiment.tex}
\input{sections/conclusion.tex}

\section*{Acknowledgements}
Yifei Zuo performed the research project partly during his internship at Snowfake AI Research. Authors would also like to thank Boyi Liu for meaningful discussions and valuable feedback.

\bibliographystyle{ims}

\input{references.bbl}
\newpage
\appendix
\input{sections/appendix.tex}

\end{document}

%% file: math_commands.tex
\usepackage{amsmath,amsfonts,bm}

\newcommand{\Bias}{\mathrm{Bias}}
\newcommand{\supp}{\mathrm{supp}}

\def\eqref#1{equation~\ref{#1}}

\def\1{\bm{1}}

\newcommand{\NW}{\mathrm{NW}}
\newcommand{\Ll}{\mathrm{LL}}
\newcommand{\Gl}{\mathrm{GL}}
\newcommand{\MSE}{\mathrm{MSE}}

\DeclareMathAlphabet{\mathsfit}{\encodingdefault}{\sfdefault}{m}{sl}
\SetMathAlphabet{\mathsfit}{bold}{\encodingdefault}{\sfdefault}{bx}{n}

\newcommand{\Var}{\mathrm{Var}}



%% file: sections/introduction.tex
\section{Introduction}
Transformer-based architectures have dominated modern AI systems and powered breakthroughs across various fields.
The core computation primitive–Softmax Attention \citep{vaswani2023attentionneed}, adaptively processes and aggregates contextual information.
Recent research on architecture innovation has proposed numerous variants of attention mechanism such as Linear Attention (LA) \citep{yang2025parallelizinglineartransformersdelta,yang2025gateddeltanetworksimproving,siems2025deltaproductimprovingstatetrackinglinear,sun2023retentivenetworksuccessortransformer,ma2024megalodonefficientllmpretraining,liu2024longhornstatespacemodels,katharopoulos2020transformersrnnsfastautoregressive,yang2024gatedlinearattentiontransformers} and state space models (SSMs) \citep{gu2022efficientlymodelinglongsequences,gu2022trainhippostatespace,gu2024mambalineartimesequencemodeling,dao2024transformersssmsgeneralizedmodels,poli2023hyenahierarchylargerconvolutional}.
While these methods offer remarkable efficiency improvements on long sequences, they often incur a performance penalty compared to Softmax Attention \citep{bick2025understandingskillgaprecurrent,jelassi2024repeatmetransformersbetter}.
Meanwhile, many of the design choices such as gating and forgetting factor \citep{yang2025gateddeltanetworksimproving,lin2025forgettingtransformersoftmaxattention,gers2000learning,yang2023gated} are often guided by heuristics or empirical results, lacking a principled understanding.
Conversely, the search for more expressive attention mechanisms, even at an additional computational cost, has been relatively under-explored.
Recently, test-time regression \citep{wang2025testtimeregressionunifyingframework} unifies the design choices of different attention variants, indicating that the attention mechanism can be viewed as a test-time optimizer for layer-specific regression problems.
A natural question arises: \emph{can we systematically improve the Softmax Attention mechanism from the perspective of test-time regression?}

In this work, we propose \emph{local linear attention (LLA)}, an upgrade to Softmax Attention derived from local linear regression.
Our contributions are summarized as follows:
\begin{itemize}
    \item We systematically analyze the design space of attention mechanisms within the test-time regression framework and propose LLA.
    We provide theoretical comparison of LLA with Softmax Attention and LA family from the perspectives of bias-variance trade-off and demonstrate its provable advantage in associative recall capability.
    \item We provide a detailed discussion on the computation of LLA and overcome the \(\Theta(n^2 d)\) and \(\Theta(nd^2)\) memory complexity with two optimizations, where \(n\) is the sequence length and \(d\) is the dimension.
    We then introduce {FlashLLA}, a blockwise and hardware-efficient algorithm to parallelize the computation on modern accelerators.
    In addition, we implement and profile a customized inference kernel with significant memory reduction.
    \item We conduct extensive experiments on synthetic tasks including test-time regression, in-context regression, associative recall and state tracking to validate the advantages and limitations of LLA compared to strong baselines.
\end{itemize}
\begin{figure*}[t]
    \centering
    \begin{subfigure}{0.42\linewidth}
        \centering
        \vspace{0pt}
        \includegraphics[width=1\linewidth]{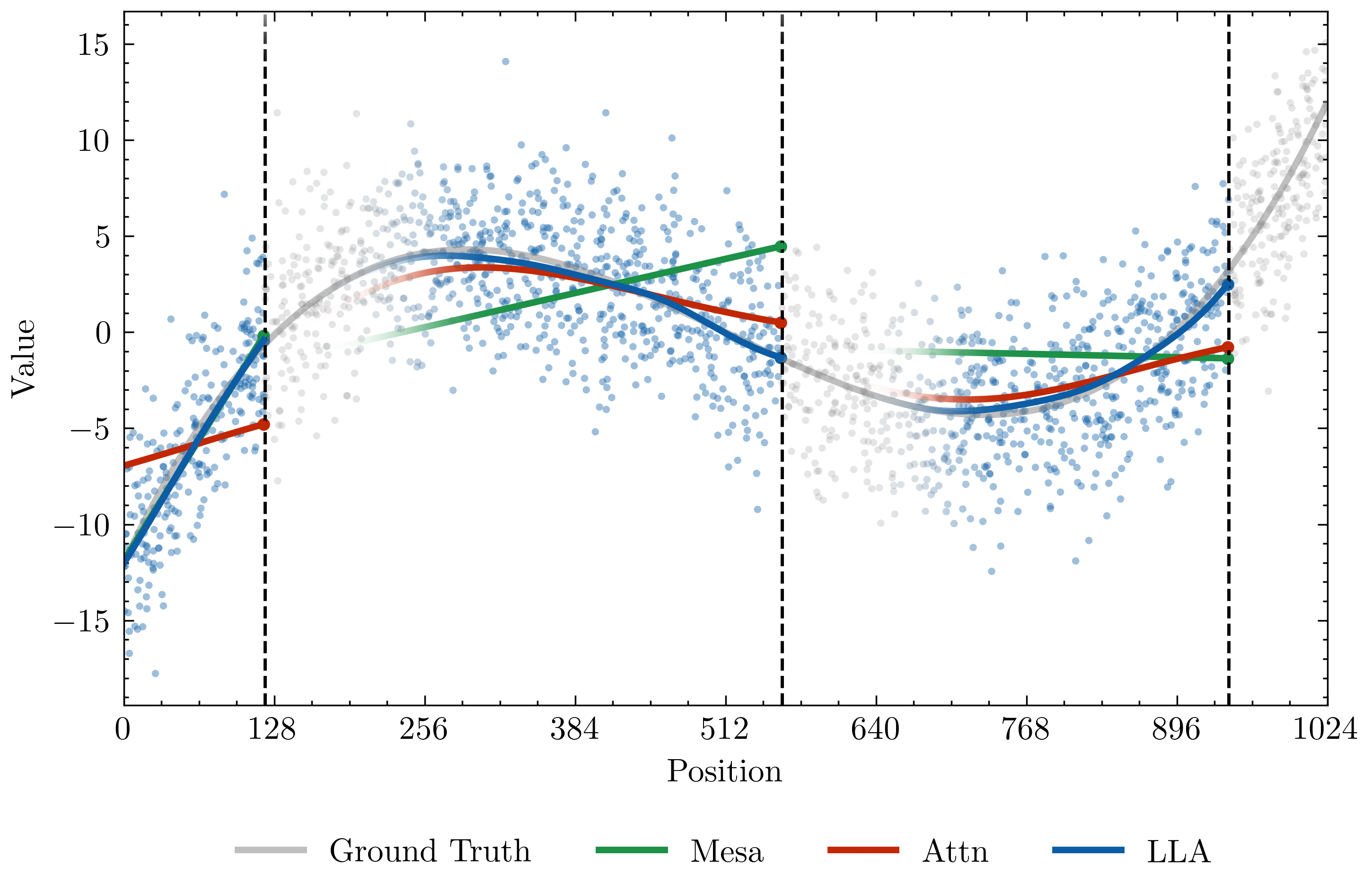}%
    \end{subfigure}%
    ~
    \begin{subfigure}{0.48\linewidth}
        \centering
        \vspace{0pt}
        \includegraphics[width=1\linewidth]{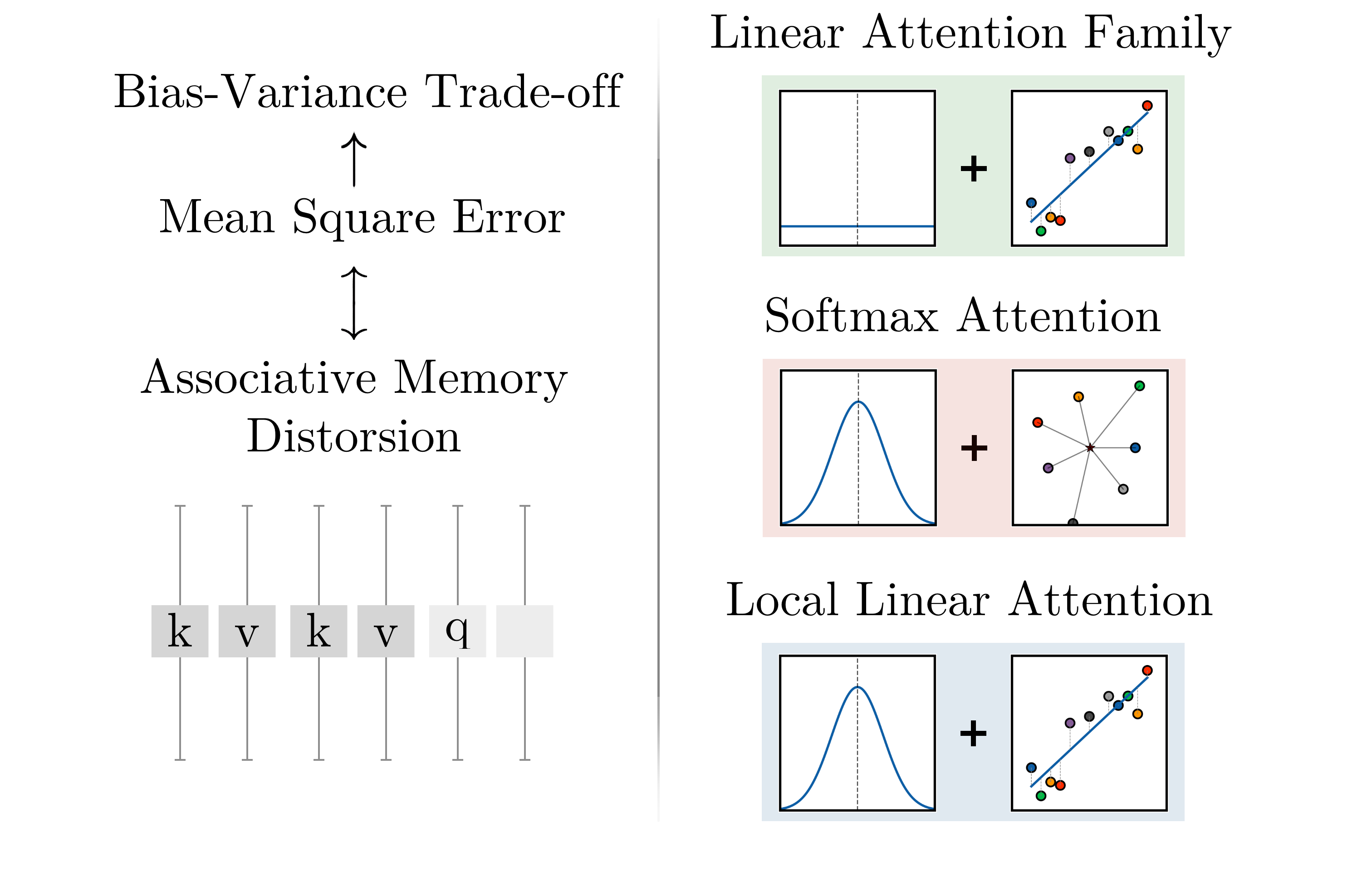}%
    \end{subfigure}
    \caption{
        A comparison of regression strategies: global linear models (e.g., SSMs, MesaNet, DeltaNet) employ query agnostic linear fits and suffer from irreducible approximation error due to model misspecification; local constant models (e.g., Softmax Attention) perform query-specific local averaging but exhibit boundary bias; local linear models (e.g., LLA) achieve a superior bias-variance trade-off combining locality and linear fitting.
    }
    \label{fig:front-lla}
\end{figure*}

\subsection{Related Works.}
\paragraph{Linear Attention and State Space Models.}
Due to the quadratic computational cost and linear memory consumption of the softmax attention mechanism for autoregressive sequence modeling, efficient attention mechanisms such as LA \citep{yang2025parallelizinglineartransformersdelta,yang2025gateddeltanetworksimproving,siems2025deltaproductimprovingstatetrackinglinear,sun2023retentivenetworksuccessortransformer,ma2024megalodonefficientllmpretraining,liu2024longhornstatespacemodels,yang2024gatedlinearattentiontransformers} and SSMs \citep{gu2022trainhippostatespace,gu2022efficientlymodelinglongsequences,gu2024mambalineartimesequencemodeling,dao2024transformersssmsgeneralizedmodels,poli2023hyenahierarchylargerconvolutional} were proposed in search of more efficient alternatives for long-context sequence generation.
These methods maintain a constant sized hidden state \citep{katharopoulos2020transformersrnnsfastautoregressive} during decoding and update it like a linear RNN.
This state update behavior is also referred as fast-weight programming \citep{6796337,schlag2021lineartransformerssecretlyfast}.
Essentially, MesaNet \citep{vonoswald2025mesanetsequencemodelinglocally} is a LA variant that preconditions the hidden state and achieves optimal regression objectives among linear models.

\paragraph{In-Context Learning by Optimization.}
A growing body of work suggests that attention mechanism implicitly performs optimization algorithms to achieve {in-context learning} \citep{garg2023transformerslearnincontextcase,akyurek2022learning,vonoswald2023transformerslearnincontextgradient,kirsch2024generalpurposeincontextlearningmetalearning,zhang2024trained,mahankali2023stepgradientdescentprovably,ahn2023transformerslearnimplementpreconditioned,dai2023gptlearnincontextlanguage}.
For example, Mesa optimization \citep{vonoswald2023transformerslearnincontextgradient} suggested that the attention layer inherently performs or approximates optimization steps during the forward pass.
This behavior is particularly evident in LA and SSMs, as the hidden state update can be interpreted as performing gradient descent to solve a linear regression objective \citep{wang2025testtimeregressionunifyingframework,liu2024longhornstatespacemodels}.
MesaNet \citep{vonoswald2025mesanetsequencemodelinglocally} is a one-step convergent algorithm for such a problem as the objective accepts a closed-form solution.

\paragraph{Hardware Efficient Attention.}
Prior works have focused on efficient attention implementation on modern hardwares to alleviate memory and computation overheads.
FlashAttention \citep{dao2022flashattentionfastmemoryefficientexact,dao2023flashattention2fasterattentionbetter} performs a block-wise online softmax to reduce I/O latency.
Several other approaches, including NSA \citep{yuan2025native}, SeerAttention \citep{gao2024seerattention}, MoBA \citep{lu2025moba}, and Block Sparse Attention \citep{xiao2025statistics}, leverage sparsity to lower the effective computational cost while preserving GPU utilization.
Flash Linear Attention \citep{yang2023gated} provides a hardware-friendly formulation of linear attention through chunk-wise computation.

\subsection{Notation.}
We use upper-case letters to denote matrices and lower-case letters to denote vectors.
For a matrix \(X\), we denote its Frobenius norm as \(\|X\|_F\) and the Hadamard product as \(\odot\).
For a vector \(x\), we denote its Euclidean norm as \(\|x\|_2\).
Furthermore, define \(\texttt{rsum}(X)=X\mathbf 1\) for matrix \(X\) and \(\texttt{bcast}(x) = x\mathbf 1^\top\) for vector \(x\), where \(\mathbf 1\) is a vector of ones.
We use the abbreviation \(\texttt{brsum}(x) = \texttt{bcast}(\texttt{rsum}(x))\).

%% file: sections/prelim.tex
\section{Beyond Local Constant Estimate}\label{section:prelim}
In this section, we first revisit the test-time regression interpretation for the attention mechanism \citep{wang2025testtimeregressionunifyingframework}. 
Then, we analyze the associative recall capacity and show the inherent limitations of LA and Softmax Attention.
Lastly, we introduce the formulation for LLA.

\subsection{Attention as Test-Time Regression}
In test-time regression framework, the attention mechanism is interpreted as a layer-specific regression solver.
The goal is to approximate an unknown regression function \(f:\mathbb R^d\mapsto\mathbb R^d\) using historical key-value pairs.
To be specific, given a hypothesis space \(\mathcal F\) and a position \(1\le i\le n\), an estimator \(\hat f_i\in\mathcal F\) is fitted on the dataset \(\mathcal D_i = \{(k_j, v_j)\in\mathbb R^d\times \mathbb R^d\}_{j=1}^i\), where the attention keys \(k_j\) serve as the features and attention values \(v_j\) as the labels.
The prediction is made at a query \(q_i\in\mathbb R^d\), which is treated as a test data point.

\paragraph{Linear Attention as Parametric Regression.}
Parametric model constrains the function class \(\mathcal F\) to a set of functions defined by a finite-dimensional parameter \(\theta\in\Theta\).
The most fundamental instantiation is the linear regression, which sets \(\mathcal F = \{f_\theta(x)=Wx+b\mid \theta=(W,b), W\in\mathbb R^{d\times d}, b\in\mathbb R^d\}\).
We omit the intercept \(b\) for simplicity. At each position \(i\), the parameter \(W_i\) is estimated by solving the following least square problem on the training dataset \(\mathcal D_i\),
\begin{align}
    \min_W \mathcal L(W;\mathcal D_i) = \frac{1}{2}\sum_{j=1}^i \gamma_{ij} \|v_j - Wk_j\|^2_2 + \lambda \|W\|^2_F,\label{eq:ls}
\end{align}
where \(\gamma_{ij}\in\mathbb R\) is a weighting factor and \(\lambda\ge 0\) is the ridge regularization penalty.
For appropriate regularization, objective \eqref{eq:ls} admits a closed-form optimal solution. MesaNet \citep{vonoswald2025mesanetsequencemodelinglocally,vonoswald2024uncoveringmesaoptimizationalgorithmstransformers} hardcodes this solution for the case \(\gamma_{ij}=1\), where the prediction is given by,
\begin{align}
    \hat f_{\texttt{Mesa}}(q_i) = \hat W_i^\texttt{Mesa}q_i = \biggl(\underbrace{\sum_{j=1}^i v_jk_j^\top}_{S_i}\biggr) \biggl(\underbrace{\sum_{j=1}^i k_jk_j^\top + \lambda I}_{H_i}\biggr)^{-1} q_i.\label{eq:mesanet}
\end{align}
Across different position \(i\), the weight \(W_i^\texttt{Mesa}\) can be updated recurrently by maintaining two statistics \(S_i\) and \(H_i\) with \(\Theta(d^2)\) memory.
This allows MesaNet to be interpreted as a linear RNN with two recurrent states. To avoid the expensive matrix inversion, vanilla LA brutally approximates the precondition matrix \(H_i\approx I\) in \eqref{eq:mesanet}, leading to a suboptimal solution of the least square problem.
It can also be shown that LA variants and SSMs such as GLA \citep{yang2024gatedlinearattentiontransformers}, RetNet \citep{sun2023retentivenetworksuccessortransformer}, RWKV \citep{peng2024eaglefinchrwkvmatrixvalued,peng2025rwkv7gooseexpressivedynamic} and Mamba \citep{gu2024mambalineartimesequencemodeling,dao2024transformersssmsgeneralizedmodels} can be derived by the approximation \(H_i\approx I\) with different weighting schemes \(\gamma_{ij}\).
Besides exact solutions, vanilla LA and variants such as DeltaNet \citep{yang2025parallelizinglineartransformersdelta} and Gated DeltaNet \citep{yang2025gateddeltanetworksimproving} can be interpreted as performing one step of first order stochastic gradient descent on weight \(W\).
We refer readers to \citep{wang2025testtimeregressionunifyingframework} and tables in \citep{peng2025rwkv7gooseexpressivedynamic,yang2025gateddeltanetworksimproving} for detailed derivations and how each model is implemented.

\paragraph{Softmax Attention is Non-Parametric.}\label{para:sa}
Non-parametric regression makes minimal structural assumptions on the function class \(\mathcal F\). A canonical example is the kernel regression, where the estimator \(\hat f\) is defined directly from the data in a way that depends on the query point.
Specifically, for a query vector \(q_i\), let \(\mathcal F(q_i)\) denote the local function class around \(q_i\). The local regression objective is defined as:
\begin{align}
    \min_{f\in\mathcal F(q_i)} \mathcal L(f;\mathcal D_i) = \frac{1}{2}\sum_{j=1}^i w_{ij} \|v_j - f(k_j)\|_2^2,\label{eq:ls_nonparam}
\end{align}where \(w_{ij}=K_h(q_i,k_j)\in\mathbb R\) is a query-dependent weight that measures the locality of the training point \(k_j\) to the query \(q_i\).
The simplest instantiation is the constant model, where \(\mathcal F(q_i) = \{f_\theta(x) = \theta\in\mathbb R^d, \forall x\in\mathbb R^d\}\).
Solving objective \eqref{eq:ls_nonparam} with this function class yields:
\begin{align}
\hat f(q_i) = \sum_{j=1}^i s_{ij}v_j,\quad s_{ij} = \frac{w_{ij}}{\sum_{j'=1}^i w_{ij'}}\label{eq:local_constant}
\end{align}
Consider an RBF kernel \(w_{ij} = \exp(-\|k_j - q_i\|^2/h)\) with bandwidth \(h=2\sqrt{d}\), the estimator \eqref{eq:local_constant} exactly recovers the softmax attention when QK normalization \citep{dehghani2023scaling,wortsman2023small,team2024chameleon} is applied,
since in this case \(w_{ij} \propto \exp(q_i^\top k_j/\sqrt{d})\) and the common constant factor cancels in the division.
It is also known as the Nadaraya-Watson (NW) kernel regression \citep{nadaraya1964,watson1964,Bierens1988} in statistics literature.
We note that QK normalization is not strictly necessary to represent practical Softmax Attention as the additional term can naturally serve as positional encoding \citep{press2022trainshorttestlong}, and the scale effectively tunes the bandwidth \(h\) in a data-dependent manner.

\subsection{Learning Behavior of Associative Recall}
Attention mechanisms are often evaluated by the associative memory capacity \citep{zhong2025understandingtransformerperspectiveassociative,behrouz2025itsconnectedjourneytesttime,ramsauer2021hopfieldnetworksneed}.
Specifically, given a training set of key-value pairs \(\{(k_j,v_j)\}_{j=1}^i\), the model is expected to retrieve the value \(v_j\) associated with \(k_j\) when queried at \(q = k_j\).
This objective can be exactly captured by the Mean Square Error (MSE). For example, the retrieval error in vanilla LA is given by,
\begin{align}
    \mathrm{MSE}_i^{\texttt{LA}} = \frac{1}{i}\sum_{j=1}^i \|S_i k_j - v_j\|^2 = \frac{1}{i}\sum_{j=1}^i\|(k_j^\top k_j-1)v_j + \sum_{j'\ne j} v_{j'}k_{j'}^\top k_j\|^2,\label{eq:associative_recall}
\end{align}
The first term in the summation is the signal bias and can be avoid by QK normalization. The second term is the interference from other key-value pairs.
Deltaformer \citep{zhong2025understandingtransformerperspectiveassociative} quantitatively analyze this error by the inverse of signal-to-noise ratio \(\text{SNR}^{-1}\), which is essentially a normalized version of MSE.
In classic literature, MSE decomposition allows us to analyze the approximation and generalization error of the model and the corresponding bias-variance trade-off.

\paragraph{Irreducible Approximation Error of Global Linear Model.}
Recall that global linear models correspond to solving a least square problem \eqref{eq:ls} over the hypothesis space \(\mathcal F = \{f_\theta(x)=Wx+b\}\).
When the ground truth function \(f\) is not global linear, any estimator \(\hat f\in\mathcal F\) will suffer from a non-vanishing {approximation error} due to model misspecification.
In contrast, the local constant model is a nonparametric estimator that does not impose structural assumptions on the function class except for smoothness or regularity.
Consequently, the approximation error vanishes asymptotically with proper assumptions. In fact, we have the following separation result between global linear (GL) and local constant (NW) estimators:
\begin{proposition}
    \label{prop:separation}
    Let \((X_i,Y_i)_{i=1}^n\) be i.i.d., \(X_i\in\mathbb{R}^d\) supported on a bounded set \(D\subset\mathbb{R}^d\), and
\(Y_i=f(X_i)+\varepsilon_i\in\mathbb{R}^{d_y}\) with \(\mathbb{E}[\varepsilon_i\mid X_i]=0\) and \(\mathbb{E}[\varepsilon_i^2\mid X_i]=\sigma^2(X_i)\).
Let \(\widehat f_{\Gl}\) denote a global-linear estimator and \(\widehat f_{\NW}\) the local-constant (NW) estimator with optimal bandwidth.
Under mild assumptions, if \(f\) is not globally linear, then
    \[
     \mathbb{E}\int_D||\widehat{f}_{\Gl}(x)-f(x)||^2dx= \Omega(1) \;,\; \mathbb{E}\int_D||\widehat{f}_{\NW}(x)-f(x)||^2dx= O(n^{-3/(d+3)}).
    \]
\end{proposition}
\paragraph{Irreducible Boundary Bias of Local Constant Model.}
Despite the appealing convergent property of local constant model, it suffers when predicting near the boundary of the data support, particularly with symmetric kernels like RBF.
This phenomenon becomes more pronounced in high-dimension and likely to occur more frequently in autoregressive prediction.
Local polynomial regression is a standard remedy in nonparametric statistics to address this issue. In Section~\ref{section:lla}, we will introduce LLA as a natural adaptation of local linear regression.
In fact, we have the following separation result between local constant (NW) and local linear (LL) estimators:
\begin{proposition}
    \label{prop:lc_and_lla}
    Under the setting of Proposition~\ref{prop:separation}, let \(\widehat f_{\NW}\) and \(\widehat f_{\Ll}\) denote local-constant and local-linear estimators with their respective optimal bandwidths. Under mild assumptions, if \(f\) has sufficiently large normal gradient along the boundary of $D$, then
    \[
    \mathbb{E}\int_D||\widehat{f}_{\NW}(x)-f(x)||^2dx= \Omega(n^{-3/(d+3)})\;, \; \mathbb{E}\int_D||\widehat{f}_{\Ll}(x)-f(x)||^2dx= O(n^{-4/(d+4)}).
    \]
\end{proposition}
The proof of proposition~\ref{prop:separation} and \ref{prop:lc_and_lla} are provided in Appendix~\ref{appendix:proof1} and \ref{appendix:proof2} correspondingly.

\subsection{Local Linear Attention}\label{section:lla}
\paragraph{Formulation. }For a query \(q_i\), instantiate the local function class \(\mathcal F(q_i) = \{f_\theta(x) = b + W(x-q_i)\mid \theta=(W,b), W\in\mathbb R^{d\times d}, b\in\mathbb R^d\}\).
The regularized local linear regression objective is,
\begin{align}
    \min_{f\in\mathcal F(q_i)}\mathcal L(f;\mathcal D_i) = \frac{1}{2}\sum_{j=1}^i w_{ij} \|v_j - b - W(k_j - q_i)\|^2 + \lambda\|W\|^2_F,\label{eq:lla_obj}
\end{align}
where \(w_{ij}\) are query-dependent kernel weights and \(\lambda \ge 0\) is a ridge penalty.
This objective also admits a closed-form solution for the intercept \(b\) and weight \(W\). Importantly, at test time, the prediction is only made at \(\hat f(q_i) = \hat b_i\). Thus it suffices to derive the formulation for the intercept.
Define \(z_{ij} = k_j - q_i\) and the following query-specific statistics,
\begin{align}\label{eq:query_specific}
    \omega_i = \sum_{j=1}^i w_{ij}\in\mathbb R,\quad \mu_i = \sum_{j=1}^i w_{ij} z_{ij}\in\mathbb R^d,\quad \Sigma_i = \sum_{j=1}^i w_{ij}z_{ij}z_{ij}^\top + \lambda I\in\mathbb R^{d\times d}.
\end{align}Also denote \(\rho_i = \Sigma_i^{-1}\mu_i \in \mathbb R^d\).
The optimal intercept can be computed as follows,
\begin{align}\label{eq:lla}
    \hat f(q_i) = \hat b_i = \sum_{j=1}^i s_{ij}v_j, \quad s_{ij} = w_{ij}\frac{1-z_{ij}^\top \rho_i}{\omega_i - \mu_i^\top \rho_i}.
\end{align}Similar to maintaining \(H_i\) in MesaNet, LLA requires capturing the precondition matrix \(\Sigma_i\).
However, a key difference is that \(H_i\) is built on global statistics that are independent of the query, whereas \(\Sigma_i\) is constructed from features centered around the specific query \(q_i\) for each position \(1\le i\le n\).
Consequently, LLA requires a KV cache of size \(\Theta(n d)\) similar to Softmax Attention, rather than constant-size recurrent states as in the LA family.
\paragraph{LLA Interpolates Linear and Softmax Attention.}
A more interpretable form of \eqref{eq:lla} can be obtained by decomposing the prediction into two components.
Suppose that the weight matrix \(\hat W_i\) is given prior to solving \eqref{eq:lla_obj}, then the optimal intercept can be expressed as,
\begin{align}
    \hat b_i = \sum_{j=1}^i s_{ij}(v_j - \hat W_i k_j) + \hat W_i q_i,\quad s_{ij} = \frac{w_{ij}}{\sum_{j'=1}^i w_{ij'}}.\label{eq:lla_decomp}
\end{align}
The first term is a local constant regression to predict the residuals \(v_j - \hat W_i k_j\), while the second term is a linear prediction based on \(\hat W_i\).
The formulation recovers LLA if \(\hat W_i\) is obtained by optimally solving \eqref{eq:lla_obj}.
However, by allowing suboptimal estimation, one can construct \(\hat W_i\) as a recurrent state similar to LA.
This decomposition reveals how LLA interpolates between Linear and Softmax Attention and provides a template for designing new algorithms.

%% file: sections/lla.tex
\section{Practical Algorithm}\label{section:algorithm}
In this section, we provide a detailed discussion of the computation involved in LLA.
We highlight two major challenges in naïve implementations and develop a practical block-wise algorithm {FlashLLA} that scales efficiently on modern accelerators such as GPUs.

\subsection{Memory Efficient Primitives}
\paragraph{Avoid Pairwise Materialization.}
The first bottleneck is the evaluation of vectors \(z_{ij}=k_j - q_i\) for every \(1\le j \le i \le n\), which requires \(\Theta(n^2d)\) memory to materialize.
This pairwise difference is later used in the formulation of \(\mu_i\) and \(\Sigma_i\) defined in \eqref{eq:query_specific} as well as the inner product \(z_{ij}^\top \rho_i\) in \eqref{eq:lla}.
For both cases, explicit materialization of \(z_{ij}\) can be avoid by algebraically separating the contributions of \(k_j\) and \(q_i\) to the final result.
Specifically, the statistics \(\mu_i\) and \(\Sigma_i\) can be reformulated in terms of intermediate quantities that are independent of the query and can then be transformed to recover the original centered statistics:
\begin{align}
    \tilde \mu_i = \sum_{j=1}^i w_{ij} k_{j}\in\mathbb R^d,\quad \tilde \Sigma_i = \sum_{j=1}^i w_{ij}k_{j}k_{j}^\top + \lambda I\in\mathbb R^{d\times d}\label{eq:tilded_form}\\
    \mu_i = \tilde \mu_i - \omega_i q_i,\quad \Sigma_i = \tilde \Sigma_i - \tilde\mu_i q_i^\top - q_i \tilde \mu_i^\top + \omega_i q_i q_i^\top\label{eq:tilded_to_original}
\end{align}
The computation in \eqref{eq:tilded_form} and \eqref{eq:tilded_to_original} only requires vectors \(k_j\) and \(q_i\) individually, reducing the memory cost to \(\Theta(nd)\).
The same principle applies to computing inner products of the form \(z_{ij}^\top x_i = k_j^\top x_i - q_i x_i\) for any vector \(x_i\in\mathbb R^d\).
The matrix operator \textbf{relative matrix multiplication (relmm)} for this optimization as follows,
\begin{align}
    \texttt{relmm}(X,Q,K) := XK^\top - \texttt{brsum}(X\odot Q).
\end{align}
This operator is invoked once in the forward computation with \(x_i=\rho_i\), but appears multiple times in the backward with other variables.
Further details on the backward are provided in Appendix \ref{appendix:backward}.

\paragraph{Matrix-Free Inversion via Conjuagte Gradients.}
The second bottleneck arises in solving linear systems of the form \(\Sigma_i^{-1}x_i\) for some vector \(x_i\in\mathbb R^d\). Directly inverting \(\Sigma_i\) for every \(1\le i \le n\) incurs a prohibitive \(\Theta(n d^2)\) memory.
Following the approach in MesaNet \citep{vonoswald2025mesanetsequencemodelinglocally}, we exploit the sum-of-rank-one structure of \(\Sigma_i\) and solve the linear system iteratively using the conjugate gradient (CG) method \citep{hestenes1952methods}.
The key insight is that CG only evaluate the matrix-vector product \(\Sigma_i p\) for a search direction \(p\in\mathbb R^d\) without explicit matrix materialization:
\begin{align}
    \Sigma_i p &= \sum_{j=1}^i w_{ij}(k_j^\top p)k_j - (q_i^\top p) \tilde\mu_i - (\tilde \mu_i^\top p) q_i  + (\omega_i q_i^\top p) q_i + \lambda p.
\end{align}
Each term only involves inner products or weighted sums over keys and the query, both of which can be computed efficiently using batched matrix multiplication with \(\Theta(nd)\) memory.
This CG operation of \(\Sigma_i\) is invoked once in the forward computation with \(x_i=\mu_i\) and twice in the backward computation with other variables.
Further details on the CG algorithm are provided in \ref{appendix:cg}.

\subsection{Parallel Form and Blockwise Algorithm}

\paragraph{Matrix Formulation.}
We first express the key components of the LLA forward pass in matrix form.
Let \(Q,K,V\in\mathbb R^{n\times d}\) be the query, key, and value matrices respectively for a given layer and head.
This function applies a causal mask to the input tensor using the \(\texttt{tril}\) operator that preserves the lower-triangular matrix.
Then the output \(O\in\mathbb R^{n\times d}\) can be computed as follows,
\begin{align}
    W = \texttt{tril}(\exp(Q K^\top/h)),&\quad M = WK - \texttt{brsum}(W) \odot Q\\
    R = \texttt{CGSolve}(M, Q, K, \lambda),&\quad \delta = \texttt{rsum}(W) - \texttt{rsum}(M\odot R)
\end{align}
\begin{align}
    O = &\left(\frac{1 - \texttt{relmm}(R,Q,K)}{\texttt{bcast}(\delta)} \odot W\right)V,
\end{align}
where \(W\in\mathbb R^{n\times n}\) is the matrix of kernel weight \(w_{ij}\), \(M\in\mathbb R^{n\times d}\) stores the first-order statistics \(\mu_i\),
\(R\in\mathbb R^{n\times d}\) contains the solution to the linear systems \(\Sigma_i^{-1}\mu_i\) for every \(1\le i \le n\), and \(\texttt{CGSolve}(\cdot)\) invokes the CG algorithm that construct \(\Sigma_i\) implicitly and solve these systems in parallel.
The division and subtraction are performed element-wise.
The single-head computation can be naturally extended to multi-head the same way as in standard multi-head attention mechanisms.

\begin{algorithm}
\caption{{FlashLLA} Forward Pass}
\begin{algorithmic}[1]\label{alg:flashlla_fwd}
\REQUIRE{Matrices \(Q,K,V\) in HBM, block sizes \(B_r, B_c\), regularization \(\lambda\), bandwidth \(h\).}
\STATE Divide \(Q\) into \(\lceil n / B_r \rceil\) blocks of size \(B_r\) and \(K,V\) into \(\lceil n / B_c \rceil\) blocks of size \(B_c\).
\STATE Divide output \(O, R\) into \(\lceil m / B_r \rceil\) blocks of size \(B_r\).
\FOR{$r = 1$ to $\lceil m / B_r \rceil$}
    \STATE Load $Q_r$ from HBM to SRAM.
    \STATE Initialize on-chip:
    \(M_r^{(0)} \leftarrow 0 \in \mathbb{R}^{B_r \times d}, \omega_r^{(0)} \leftarrow 0 \in \mathbb{R}^{B_r}, m_r^{(0)} \leftarrow -\infty \in \mathbb{R}^{B_r}\)
    \FOR{$c = 1$ to $\lceil n / B_c \rceil$}\label{ln:iter1start}
        \STATE Load $K_c$ from HBM to SRAM.
        \STATE Compute $W = Q_r K_c^\top/h$ and $m = \max(m_r^{(c-1)}, \texttt{rowmax}(W))$.
        \STATE Compute $\alpha_r = \exp(m_r^{(c-1)} - m)$, $W = \exp(W - \texttt{bcast}(m))$ and update $m_r^{(c)} = m$.
            \STATE Compute $\omega_r^{(c)} = \alpha_r^{(c)} \odot \omega_r^{(c-1)} + \texttt{rsum}(W)$
        \STATE Compute $M_r^{(c)} = \texttt{bcast}(\alpha_r^{(c)}) \odot M_r^{(c-1)} + W K_c$.
    \ENDFOR\label{ln:iter1end}
    \STATE Initialize on-chip: \(O_r^{(0)} \leftarrow 0 \in \mathbb{R}^{B_r \times d}, R_r^{(0)} \leftarrow 0 \in \mathbb{R}^{B_r \times d}\)
    \STATE Compute $M_r = M_r^{(\texttt{last})} - \texttt{brsum}(W)\odot Q_r$.
    \STATE Compute $R_r = \texttt{CGSolve}(M_r, Q_r, K, M_r^{(\texttt{last})}, \omega_r^{(\texttt{last})}, \lambda$).\label{ln:iter2}
    \STATE Compute $\delta_r = \omega_r^{(\texttt{last})} - \texttt{rsum}(M_r \odot R_r)$.
    \FOR{$c = 1$ to $\lceil n / B_c \rceil$}\label{ln:iter3start}
        \STATE Load $K_c, V_c$ from HBM to SRAM.
        \STATE Compute $W = \exp(Q_r K_c^\top / h - \texttt{bcast}(m_r^{(\texttt{last})}))$.
        \STATE Compute $S = (1-\texttt{relmm}(R_r,Q_r,K_c)) \odot W / \texttt{bcast}(\delta_r)$.
        \STATE Compute $O_r^{(c)} = O_r^{(c-1)} + S V_c$.
    \ENDFOR\label{ln:iter3end}
\ENDFOR
\end{algorithmic}
\end{algorithm}

\paragraph{Blockwise Algorithm.}
Denote \(B_r,B_c\) as the block size for queries and keys/values along the sequence length dimension and \(r,c\) as the block index.
Denote \(Q_r\in\mathbb R^{B_r\times d}\) and \(K_c, V_c \in \mathbb R^{B_c\times d}\) as the block-wise representations.
The forward pass of {FlashLLA} is summarized in Algorithm~\ref{alg:flashlla_fwd}.

\begin{wrapfigure}[14]{r}{0.3\textwidth}
  \centering
  \includegraphics[width=0.3\textwidth]{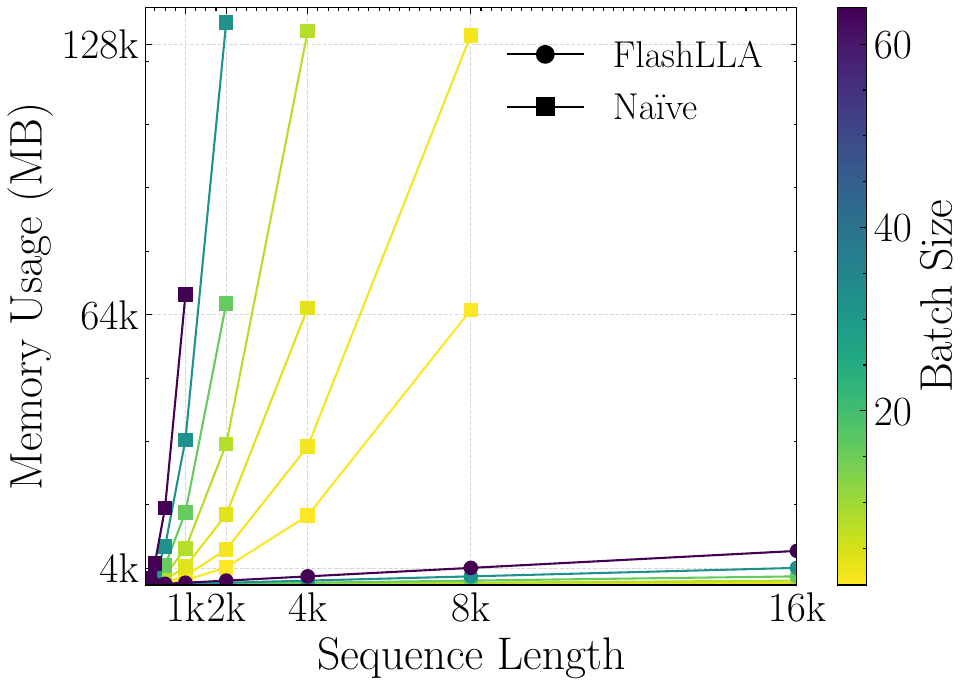}
  \caption{FlashLLA reduces the working set memory to \(\Theta(nd)\). The figure shows the profiling result for \(d=128\), OOM points are omitted.}
  \label{fig:memory}
\end{wrapfigure}
Since the statistics \eqref{eq:query_specific} for each query are computed independently, the forward pass of LLA can be naturally made parallel for batched queries.
Therefore, the algorithm proceeds by iterating over query blocks \(r\).
Within each query block, the iteration over key/value blocks has three passes.
(i) The first pass (line~\ref{ln:iter1start}-\ref{ln:iter1end}) corresponds to accumulating the statistics \(M_r, \omega_r\) in an online fashion. Similar to online softmax \citep{ye2023online,milakov2018onlinenormalizercalculationsoftmax}, we maintain a running maximum \(m_r\) to ensure numerical stability when computing the kernel weights.
This trick is valid as the computation \eqref{eq:lla} is homogeneous in \(w_{ij}\).
(ii) The second pass (line~\ref{ln:iter2}) is encapsulated in the \(\texttt{CGSolve}(\cdot)\) operator (see Appendix~\ref{appendix:cg} for details).
(iii) The third pass (line~\ref{ln:iter3start}-\ref{ln:iter3end}) computes the final output \(O_r\) using the pre-computed results and the values \(V_c\). To save computation in backward pass, we also store the intermediate \(R_r\) and denominator \(\delta_r\) alongside the output into HBM.

We implement and benchmark the algorithm~\ref{alg:flashlla_fwd} in a custom \texttt{Triton} kernel (\(\sim\)500 lines of Python) across a range of dimensions and batch sizes on a single NVIDIA H200 GPU.
Figure~\ref{fig:memory} demonstrate the quadratic dependency and quickly runs out of memory for naïve method.
In contrast, the blockwise \texttt{Triton} kernel significantly reduces the working set and scales linearly with sequence length, making it hardware-efficient and feasible for long-context and large batch training or inference.


%% file: sections/experiment.tex
\section{Empirical Results}\label{section:experiment}
\paragraph{Test-Time Regression on Non-Stationary Data.}\label{section:regression_fit}
\begin{figure}
    \centering
    \includegraphics[width=0.9\linewidth]{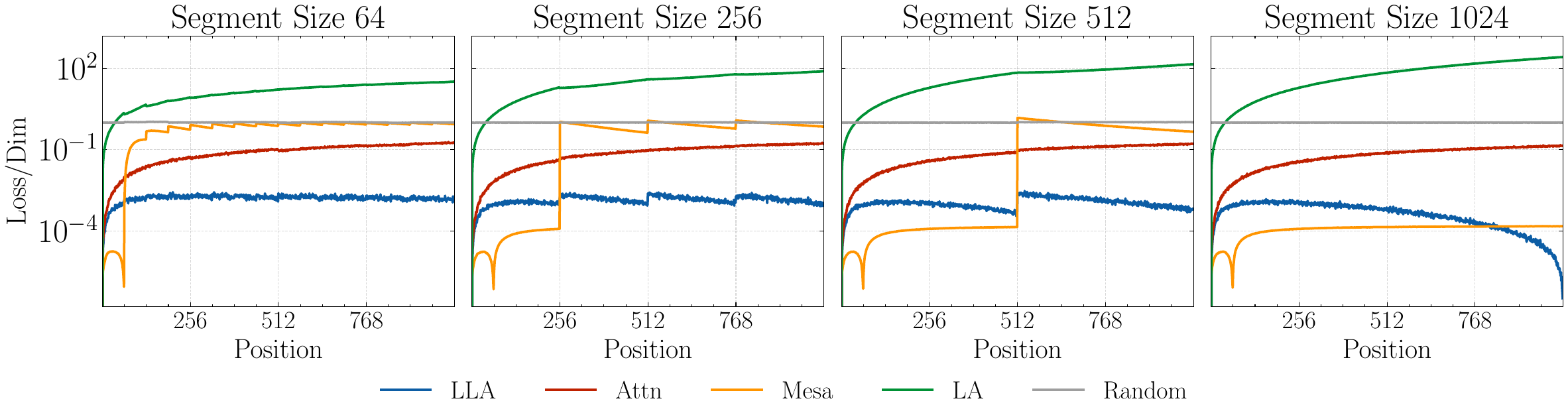}
    \caption{
        Test-time regression performance on a piecewise-linear task.
        The figures demonstrate position-wise MSE for \(d=64\) with \(S\in\{64,256,512,1024\}\).
        Results are averaged over \(10{,}000\) independently sampled sequences;
        LLA outperforms other baselines and benefits from more in-segment data; MesaNet excels only before the first shift. The y-axis uses a logarithmic scale.
    }
    \label{fig:regression_fit}
\end{figure}
We first devise a synthetic piecewise-linear regression task to isolate the test-time adaptation capabilities of different attention mechanisms directly without training the query, key and value projections.
Each sample is a length-\(L\) sequence partitioned into \(L/S\) contiguous segments with \(S\) being the segment size.
For each \(c\in\{1,\ldots, L/S\}\), keys \(k_i\in\mathbb R^d\) are drawn from a segment-specific distribution \(P_c\) supported on a distinct cone in the input space (see Appendix \ref{appendix:exp} for construction details).
The corresponding values \(v_i\) are generated by a segment-specific linear function \(v_i = A_ck_i + \epsilon_i\) for \(i \in \{(c-1)S+1,\ldots, cS\}\) where \(A_c \sim \mathcal N(0, I)\) and \(\epsilon_i \sim \mathcal N(0, \delta^2 I)\).
This design ensures the generated data \(\{(k_i,v_i)\}_{i=1}^L\) has non-stationary input distribution and conditional mapping \(f_c(k)=A_c k\).

\begin{wrapfigure}[14]{r}{0.3\textwidth}
  \centering
  \includegraphics[width=0.3\textwidth]{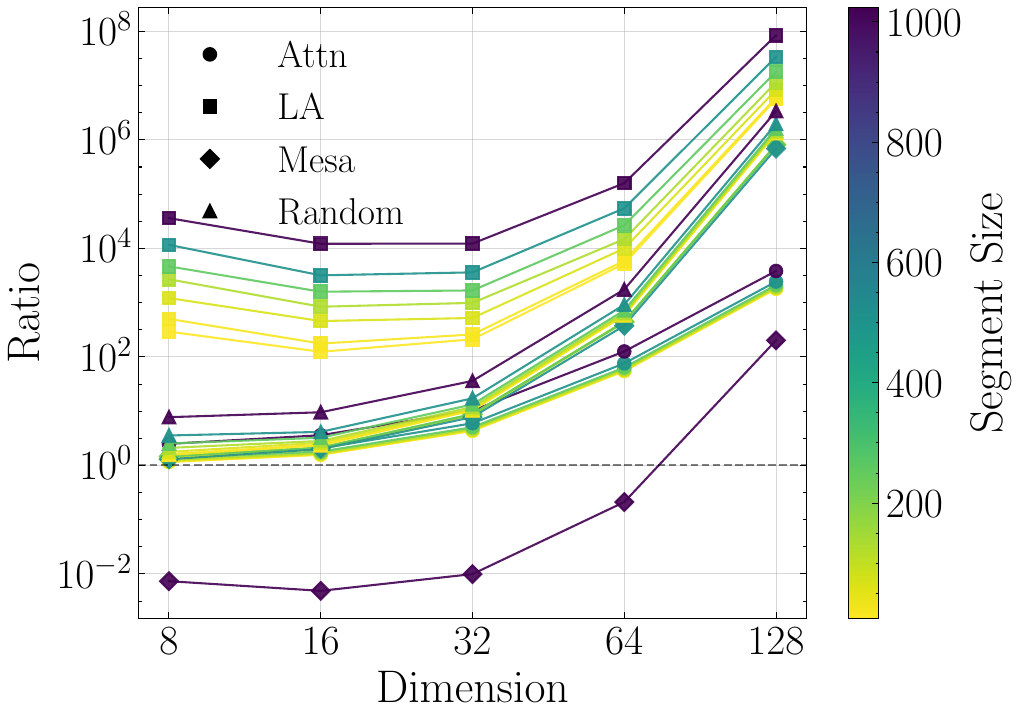}
  \caption{The advantage of LLA scales with the data dimension. Axes are in logarithmic scale.}
  \label{fig:scaling}
\end{wrapfigure}
We evaluate a single layer of each candidate model, including LLA, Softmax Attention, vanilla LA \citep{katharopoulos2020transformersrnnsfastautoregressive}, MesaNet \citep{vonoswald2025mesanetsequencemodelinglocally} as well as a random predictor.
As this is a test-time only evaluation, we exclude mechanisms that require training to adapt.
We set \(L=1024\) and sweep over different segment sizes \(S\) and input dimensions \(d\).
Performance is measured by the position-wise MSE \(\ell_i = \|\hat f(k_i)-v_i\|_2^2\) for \(i\in\{1,\ldots,L\}\) to capture the adaptation capability along the sequence.
We also investigate the scaling behavior by evaluating the MSE ratio \(\sum_{j=1}^L \ell_j^{\texttt{Model}}/\sum_{j=1}^L \ell_j^{\texttt{LLA}}\) for each model compared to LLA across different values of \(d\) and \(S\).
The results are summarized in Figure~\ref{fig:regression_fit}. 
LLA consistently outperforms other mechanisms as more non-stationarity is observed, even though MesaNet achieves higher performance in the first segment.
Meanwhile, LLA continues to improve within each segment whereas Softmax Attention does not benefit from more in-distribution data.
Moreover, the advantage of LLA scales favorably with data dimensionality (Figure \ref{fig:scaling}), indicating the potential for adaptation to larger models and datasets.

\paragraph{In-Context Regression on Non-Stationary Data.}
We next evaluate the models' ability to perform in-context regression on non-stationary, piecewise-linear data.
The data generation process follows the same principle as in the test-time regression task.
The data points \(\{x_i\in\mathbb R^{d_x}\}\) are generated from segment-specific distributions \(P_c\) and the target is given by \(y_i = A_cx_i + \epsilon_i \in\mathbb R^{d_y}\) for \(i \in \{(c-1)S+1,\ldots, cS\}\), where \(A_c \sim \mathcal{N}(0,I_{d_y\times d_x}/d_x)\) and \(\epsilon_i \sim \mathcal N(0, \delta^2 I_{d_y})\).
Query \(x^\prime\in\mathbb R^{d_x}\) is randomly sampled from the segment distributions \(P_c\). Each in-context regression prompt \(X\) is constructed by concatenating \(L\) shuffled input-target pairs with \(L^\prime\) queries:
\begin{align}
  X = \begin{pmatrix}
    x_1 & x_2 & \cdots & x_L & x_1^\prime & \cdots & x^\prime_{L^\prime}\\
    y_1 & y_2 & \cdots & y_L & 0 & \cdots & 0
  \end{pmatrix}
  \in \mathbb R^{(d_x+d_y)\times (L+L^\prime)}.
\end{align}
In contrast to the test-time regression setting, the query, key and value projections are parameterized.
The model \(f_\theta:\mathbb R^{d_x+d_y}\to\mathbb R^{d_y}\) is trained to predict the target \(Y\), where the label to each query is generated by \(y^\prime_{i}=A_c x_{i}^\prime\).
Specifically, for a dataset \(\mathcal D_{\texttt{train}} = \{(X^{(b)}, Y^{(b)})\}_{b=1}^B\), we minimize the MSE loss \(\mathcal L(\theta;\mathcal D_{\texttt{train}})\) on the query tokens and report the test error \(\mathcal L(\theta^\ast;\mathcal D_{\texttt{test}})\) after training.

We compare LLA against several strong baselines, including Softmax Attention, Mamba \citep{gu2024mambalineartimesequencemodeling,dao2024transformersssmsgeneralizedmodels}, GLA \citep{yang2024gatedlinearattentiontransformers}, Hyena \citep{poli2023hyenahierarchylargerconvolutional} and Gated DeltaNet \citep{yang2025parallelizinglineartransformersdelta,yang2025gateddeltanetworksimproving}.
We fix the dimension \(d_x=d_y=32\) and query number \(L^\prime=16\) for all sweeps over segment sizes and evaluate two-layer models without MLPs.
The results in Figure~\ref{fig:icl}\subref{fig:ic-regression} demonstrate similar trends as in the test-time regression task, where LLA consistently outperforms other baselines across all configurations, particularly with smaller segment sizes.
\paragraph{In-Context Associative Recall.}
In-context recall is a fundamental capability of language models which requires the model to retrieve relevant information from the context based on the query.
We adopt the MQAR task in Zoology \citep{arora2023zoologymeasuringimprovingrecall} to evaluate this ability.
Specifically, given two alphabets \(\mathcal A_k,\mathcal A_v\) and a set of key-value pairs \((k_i,v_i)\in\mathcal A_k\times \mathcal A_v\), the set \(\{k_i\mapsto v_i\}\) defines a many-to-one key-value association.
The model is prompted with a sequence of key-value pairs and then queried with keys sampled from the context to predict the corresponding value:
\[
\underbrace{k_1, v_1, k_2, v_2, \ldots, k_N, v_N}_{\text{context}}\ \underbrace{k_{i_1}, o_{i_1}, \ldots, k_{i_2}, o_{i_2}, \ldots, k_{i_N}, o_{i_N}}_{\text{query + answer}},
\]where \(i_1, \ldots, i_N\) is a permutation of \(\{1,\ldots,N\}\) and \(o_{i}\) is the model's prediction.
The position of query follows a power-law to simulate the realistic language data distribution.

As indicated in \citep{wang2025testtimeregressionunifyingframework}, a single short convolution layer is sufficient to solve next-token recall tasks.
Therefore, we disable short convolution in all baselines to ensure fair comparison. We set \(\|\mathcal A_k\cup \mathcal A_v\|=8k\) and sweep over different sequence lengths and number of KV pairs.
The test recall accuracy results in Figure~\ref{fig:icl}\subref{fig:ic-recall} indicate that the advantages of LLA can be effectively transferred to discrete token prediction tasks.
Additionally, we also observe different learning dynamics between LLA and Gated DeltaNet in this task. The results are discussed in Appendix~\ref{appendix:add_results}.
\begin{figure*}[t]
    \centering
    \begin{subfigure}[t]{0.58\textwidth}
        \centering
        \vspace{1pt}
        \includegraphics[width=\linewidth]{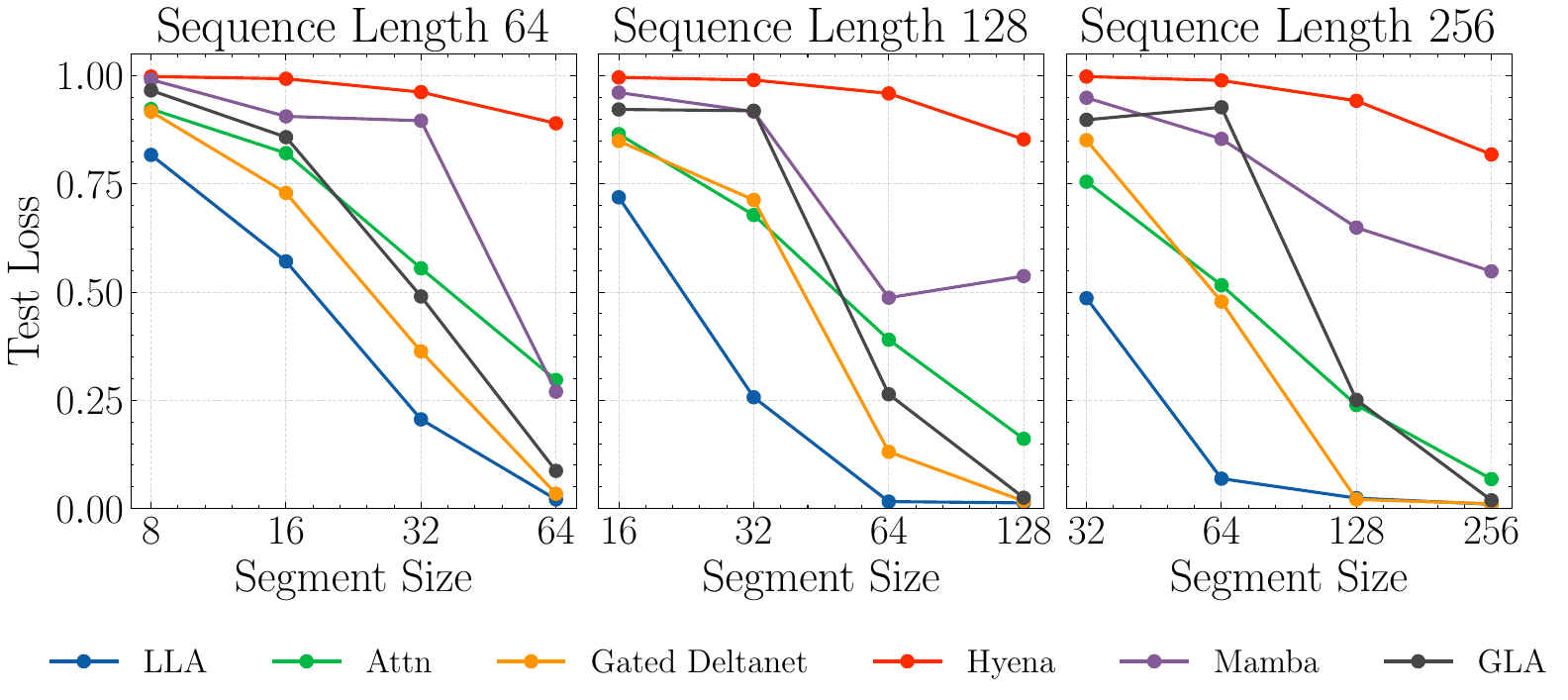}
        \caption{Test Error of In-Context Regression.}
        \label{fig:ic-regression}
    \end{subfigure}
    ~
    \begin{subfigure}[t]{0.40\textwidth}
        \centering
        \vspace{0pt}
        \includegraphics[width=\linewidth]{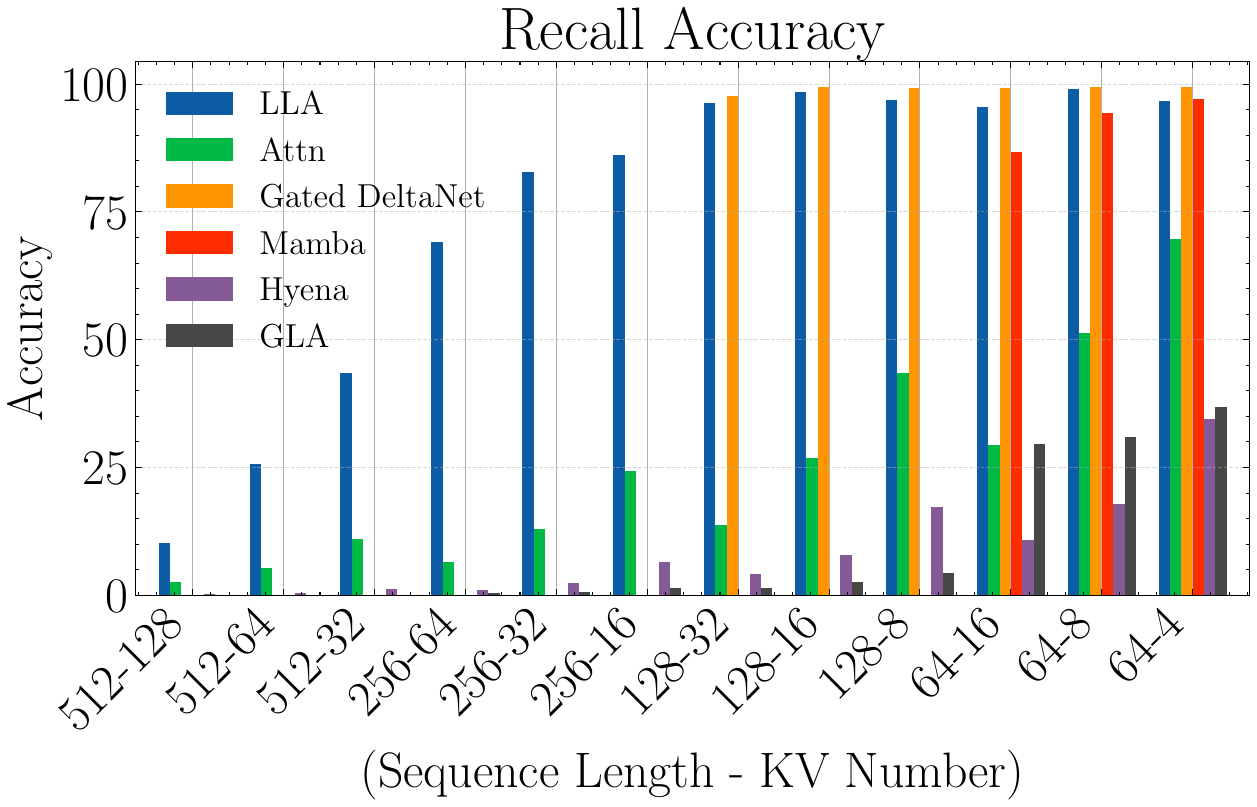}
        \caption{Test Accuracy of Associative Recall.}
        \label{fig:ic-recall}
    \end{subfigure}
    \caption{
        Figure~\subref{fig:ic-regression} and \subref{fig:ic-recall} shown for models with \(d=128\) and \(2\) attention heads. Each point represents the best performance achieved across training hyperparameters, averaged over \(3\) random seeds.
        LLA consistently outperforms other baselines in in-context regression similar to the test-time regression task, and achieves the highest accuracy in associative recall across different sequence lengths and number of key-value pairs.
    }
    \label{fig:icl}
\end{figure*}

\paragraph{Permutation State-Tracking.} We then test models' state-tracking ability by permutation state-tracking task.
Given an initial assignment of items to positions, a sequence of swap instructions, and query positions.
The model is trained to predict the item at each query position after all swaps.
Each example is constructed as
\[
\underbrace{p_1{=}\,a_1,\,p_2{=}\,a_2,\,\ldots,\,p_N{=}\,a_N}_{\text{initial state}}
\ \ \texttt{\#}\ \ \underbrace{i_1\,j_1,\,i_2\,j_2,\,\ldots,\,i_S\,j_S}_{\text{swap instruction}}
\ \ \texttt{\#}\ \ \underbrace{q_1{=}\,o_1,\,\ldots q_{N'}{=}\,o_{N'}}_{\text{query + answer}}.
\]

\begin{wrapfigure}[15]{r}{0.3\textwidth}
  \centering
  \includegraphics[width=0.3\textwidth]{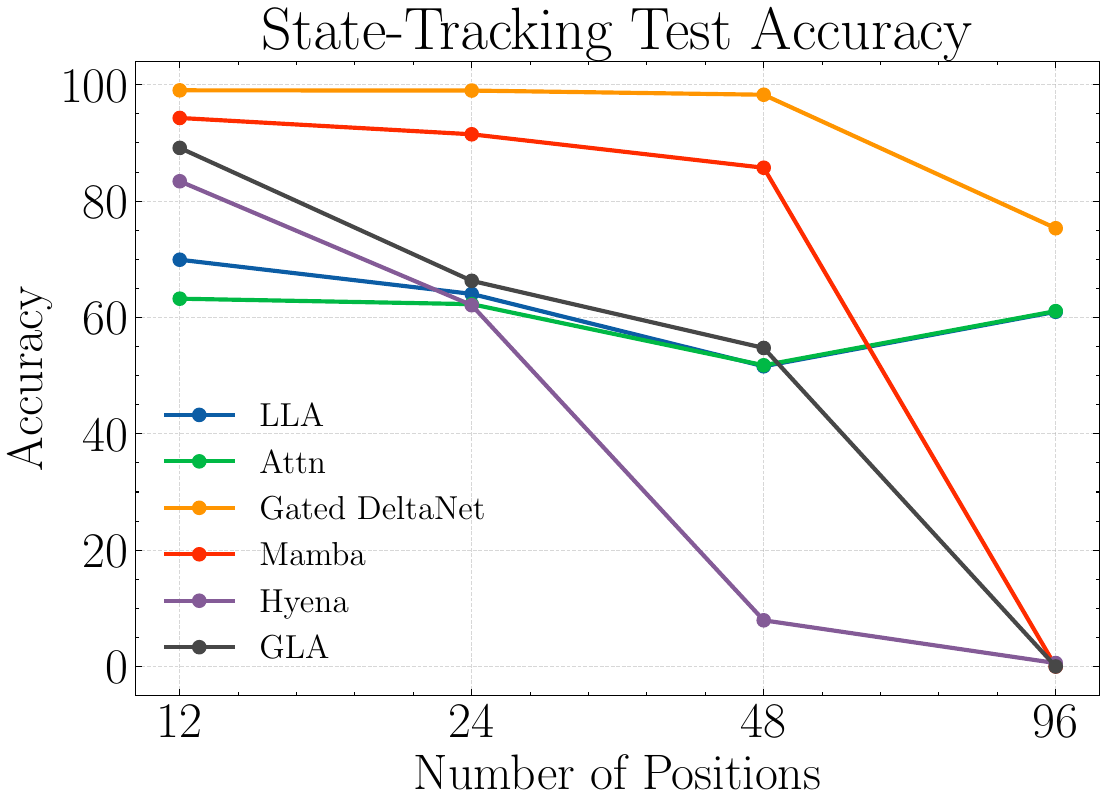}
  \caption{Test Accuracy of State Tracking. Results show the best accuracy averaged over 3 random seeds.}
  \label{fig:state_tracking}
\end{wrapfigure}
Here \(p_n\in\{1,\ldots,N\}\) denotes position \(n\); \(a_n\in\mathcal{A}\) is the item initially assigned to \(n\); each \((i_s,j_s)\in\{1,\ldots,N\}^2\) is a swap instruction exchanging the items at positions \(i_s\) and \(j_s\); and \(q\) is the queried position.
The target is the item at position \(q\) after applying all \(S\) swaps.
We include explicit delimiter tokens \(\texttt{\#}\), \(\texttt{=}\), and \(\texttt{,}\) for structure.

We draw the number of swaps as \(S\sim\text{Uniform}(N/6,N/3)\) and set \(|\mathcal A|=8k\) for each example.
The results in Figure~\ref{fig:state_tracking} show that LLA achieves test accuracy on par with Softmax Attention across $N$.
This outcome is expected from a complexity-theoretic perspective: constant-depth Softmax Attention is no more expressive than constant-depth threshold circuits $\mathsf{TC}^0$ and has limited ability to realize unbounded-depth state-tracking as $N$ grows \citep{hahn2020theoretical,merrill2023parallelism}.
By Eqs.~\ref{eq:query_specific} and \ref{eq:lla}, LLA augments Softmax Attention with a query-specific first-order correction computed via a constant number of parallel algebraic passes (weighted sum, inner product, inverse), which adds at most a constant extra circuit layer.
Its performance therefore matches the theoretical limits of Softmax Attention, explaining the results in Figure~\ref{fig:state_tracking}.

%% file: sections/conclusion.tex
\section{Limitations and Future Directions}
\paragraph{High Computation and I/O Intensity.}
Despite the significant reduction in memory consumption, LLA's computational cost remains substantially higher than that of Softmax Attention, primarily due to the matrix inversion involved in the computation.
Exploring approximations to reduce the computation is an important direction for future work.
Furthermore, while {FlashLLA} achieves the same \(\Theta(nd+n^2)\) I/O complexity as in FlashAttention when the number of CG iterations is set as a constant (which is sufficient in practice).
the constant factor is still higher due to the additional reads and writes required by the iterative solver.
Incorporating hardware-aware optimizations, such as sliding windows or sparsity, could further reduce I/O complexity.

\paragraph{Kernel Development and Evaluation on LLMs.}
This work evaluates LLA on synthetic and moderate-scale tasks; its efficacy on large language models remains an ongoing question.
Training LLMs with LLA using \texttt{PyTorch} implementation is infeasible due to its high computational and memory complexity.
Therefore significant engineering efforts are required to stabilize and optimize the forward and backward kernel.
Additionally, the numerical sensitivity of the matrix inversion poses a challenge for developing low precision kernels without sacrificing performance.

\paragraph{Efficient Interpolation of Linear and Softmax Attention.}
As shown in \eqref{eq:lla_decomp}, LLA provides an optimal interpolation between Linear and Softmax Attention in solving the regression objective.
And the formulation also provides a template to design algorithm for better computational efficiency while still retain strong estimation capabilities and potentially even improve upon the circuit complexity of Softmax Attention.
For instance, future work could explore the integration of state-of-the-art Linear Attention architectures such as DeltaNet and Mamba using this template.

%% file: sections/appendix.tex
\section{Appendix: Proof of Proposition \ref{prop:separation}}\label{appendix:proof1}
\label{appendix: proof sepa}

Let $(X_i,Y_i)_{i=1}^n$ be i.i.d.\ with $X_i\in\mathbb{R}^d$ supported on a bounded domain $D\subset\mathbb{R}^d$ with density $p$, and
\[
Y_i \;=\; f(X_i)+\varepsilon_i\in\mathbb{R}^{d_y},\qquad \mathbb{E}[\varepsilon_i\mid X_i]=0,\ \ \mathbb{E}[\varepsilon_i^2\mid X_i]=\sigma^2(X_i).
\]

The global linear estimator at $x\in D$ is
\[
\widehat{f}_\Gl(x)=\widehat{\theta}^\top(1,x^\top)^\top,
\]
where $\widehat{\theta}\in\mathbb{R}^{d+1}$ minimizes the empirical squared loss over the global affine class.

Let $K:\mathbb{R}^d\to[0,\infty)$ be a bounded, compactly supported, radially symmetric kernel with $\int K(u)\,du=1$. For a symmetric p.d.\ bandwidth matrix $H=H_n\succ 0$ with a constant condition-number upper bound $\kappa_1$, write
\[
K_H(u) \;:=\; |H|^{-1/2}\,K(H^{-1/2}u),\qquad \|H\|\to 0,\qquad n|H|^{1/2}\to\infty .
\]
The NW estimator at $x\in D$ is
\[
\widehat f_{\NW}(x)
\;=\; \frac{\bm 1^\top W \bm Y}{\bm 1^\top W \bm 1},
\]
where $W:=\mathrm{Diag}\!\big(K_H(X_i-x)\big)$, $\bm 1:=(1,\ldots,1)^\top\!\in\mathbb{R}^n$ and $\bm Y:=(Y_1,\ldots,Y_n)\in\mathbb{R}^{n\times d_y}$.

To make the analysis easier, we assume the following smoothness requirements.
\begin{assumption}
\label{asp:smooth boundary}
     The domain $D$ has $C^2$ boundary (defined as $\partial D$) with principal curvatures uniformly bounded by $\kappa_2$.
\end{assumption}
\begin{assumption}
\label{asp: smooth func}
    The density function $p$ of $X$ satisfies $p\in C^1(D)$. For all dimensions $j\in\{1,\ldots,d_y\}$, the function $f$ satisfies $f_j\in C^2(D)$,  and the variance function $\sigma^2$ satisfies $\sigma^2_j\in C(D)$.
\end{assumption}

We also assume that the kernel $K$ has easy-to-handle support.
\begin{assumption}
\label{asp:ball k}
    $K$ is radial, compactly supported in the unit ball $\mathbb{B}^d:=\{u\in\mathbb{R}^d:||u||\leq 1\}=:\supp(K)$.
\end{assumption}

Throughout, $\mathbb{E}[\cdot]$ denotes expectation with respect to the randomness of the training sample $\mathcal{S}_n:=\{(X_i,Y_i)\}_{i=1}^n$, while $\int_D(\cdot)\,dx$ denotes the Lebesgue integral over the spatial domain $D$. Because
\[
\mathbb{E}\!\int_D\|\widehat{f}(x)-f(x)\|^2\,dx
\;=\;\sum_{j=1}^{d_y}\mathbb{E}\!\int_D\bigl(\widehat{f}_j(x)-f_j(x)\bigr)^2\,dx,
\]
the output dimension $d_y$ only induces a summation across components and does not affect the order; hence, without loss of generality, we take $d_y=1$ below.

\subsection{Integral Error Estimation of Global Linear Regression}

We consider the global affine class
\[
\mathcal G \;:=\; \bigl\{\, g_\theta(x)=\beta_0+\beta^\top x \;:\; \theta=(\beta_0,\beta^\top)^\top\in\mathbb{R}^{1+d}\,\bigr\}.
\]

\begin{lemma}
\label{lem:GLR-const-lb}
If $f\notin\mathcal{G}$, there exists a constant $A_D^*>0$ such that for every $n$,
\[
\mathbb{E}\!\left[\;\int_D \bigl(\widehat f_{\Gl}(x)-f(x)\bigr)^2\,dx\;\right]
\;\ge\; A_D^\star.
\]
\end{lemma}

\begin{proof}
Let $L^2(D,dx):=\{h:D\to\mathbb{R}\ \text{with}\ \int_D h(x)^2\,dx<\infty\}$ endowed with inner product $\langle h_1,h_2\rangle=\int_D h_1(x)h_2(x)\,dx$. The set $\mathcal G$ is a finite-dimensional linear subspace and hence closed in $L^2(D,dx)$. By the Projection Theorem, the $L^2(D,dx)$-orthogonal projection of $f$ onto $\mathcal G$ exists and is unique:
\[
g^\dagger \in \arg\min_{g\in\mathcal G}\; \|f-g\|_{L^2(D)}^2
\;=\; \arg\min_{g\in\mathcal G}\; \int_D \bigl(f(x)-g(x)\bigr)^2\,dx.
\]
Set $A_D^\star:=\|f-g^\dagger\|_{L^2(D)}^2$. If $f\notin\mathcal G$, then $f-g^\dagger\neq 0$ in $L^2(D)$ and thus $A_D^\star>0$.

Fix an arbitrary realization $\mathcal S_n:=(X_i,Y_i)_{i=1}^n$. Since $\widehat f_{\Gl}(\cdot;\mathcal S_n)\in\mathcal G$, the optimality of $g^\dagger$ yields
\[
\int_D \bigl(\widehat f_{\Gl}(x;\mathcal S_n)-f(x)\bigr)^2\,dx
\;\ge\;
\inf_{g\in\mathcal G}\int_D \bigl(g(x)-f(x)\bigr)^2\,dx
\;=\; A_D^\star.
\]
This inequality holds for every sample $\mathcal S_n$. Taking expectation over the training data proves the claim.
\end{proof}

\subsection{Point-wise Error Estimation of Local Constant Regression}
In this section we will estimate the point-wise mean-squared-error of NW estimator, whose expressions are given by
\begin{equation*}
    \begin{aligned}
        \MSE_\NW(x)&:=\mathbb{E}\left[\left(\widehat{f}_\NW(x) - f(x)\right)^2\mid \{X_i\}_{i=1}^n\right]\\
        &=\underbrace{\left(\mathbb{E}\left[\widehat{f}_\NW(x) - f(x)\mid \{X_i\}_{i=1}^n\right]\right)^2}_{\Bias_\NW(x)^2}+\underbrace{\mathbb{E}\left[\left(\widehat{f}_\NW(x)-\mathbb{E}\left[\widehat{f}_\NW(x)\mid\{X_i\}_{i=1}^n\right]\right)^2\mid\{X_i\}_{i=1}^n\right]}_{\Var_\NW(x)}.
    \end{aligned}
\end{equation*}

When estimating at $x\in D$, the kernel $K_H$ maps the translated domain $(D-x)$ into $\supp(K)$ via the scaling $H^{-1/2}$. The mapped set differs depending on whether $x$ lies in the interior of $D$ or near $\partial D$. This geometric difference drives a larger boundary error for $\NW$, which will underlie the performance gap between $\NW$ and $\Ll$. We formalize the mapped kernel domain and the boundary layer.

\begin{definition}[Exact kernel domain and boundary layer.]
For any $x\in D$ and bandwidth $H$, the exact kernel domain is
\[
D_{x,H}:=H^{-1/2}(D-x)\cap \supp(K)=\{u\in\mathbb{B}^d: x+H^{1/2}u\in D\}.
\]
Define the boundary layer by
\[
\mathcal B(H):=\{\,x\in D:\ D_{x,H}\neq \supp(K)\,\}.
\]  
\end{definition}

We use the following kernel-moment shorthands.

\begin{definition}[Exact kernel moments] For all $x\in D$ and bandwidth $H$, we define
    \[
\mu_0^\star(x,H):=\int_{D_{x,H}} K(u)\,du,\qquad
\mu_1^\star(x,H):=\int_{D_{x,H}} u\,K(u)\,du,\qquad
\mu_2^\star(x,H):=\int_{D_{x,H}} uu^\top K(u)\,du.
\]
Define normalized moments $\bar\mu_r^\star(x,H):=\mu^\star_r(x,H)/\mu^\star_0(x,H)$.
\end{definition}

Then we are ready to estimate the bias and variance using $H$ and $n$.
\begin{lemma}
\label{lem: point bias}
    Under Assumption \ref{asp: smooth func}, we have
    \begin{equation*}
\Bias_\NW(x)=\nabla f(x)^\top H^{1/2}\bar\mu_1^\star(x,H)
\;+\; O_p(\|H\|),\quad \Var_{\NW}(x)=\Theta_p(\frac{1}{n|H|^{1/2}})
\end{equation*}
uniformly for all $H\in\mathcal{H}_n$, where $\mathcal{H}_n:=\{H=h^2B: h\in[n^{-a}, n^{-b}], B\succ 0, |B|=1, \kappa(B)\leq \kappa_1\}$ and $0<b<a<1$.
\end{lemma}
\begin{proof}
Defining $\bm f:=(f(X_1),\ldots f(X_n))^\top$, for any fixed point $x_0\in D$, we have
\begin{equation*}
\begin{aligned}
    \Bias_\NW(x_0) &= (\bm 1^\top W\bm 1)^{-1} \bm 1^\top W (\bm f - f(x_0)\bm 1).
\end{aligned}
\end{equation*}
Specifically, 
\begin{equation*}
\begin{aligned}
    n^{-1}(\bm 1^\top W \bm 1)= n^{-1}\sum_{i=1}^{n} K_H(X_i-x_0),\quad n^{-1}\bm 1^\top W (\bm f - f(x_0)\bm 1)= n^{-1}\sum_{i=1}^{n} K_H(X_i-x_0)(f(X_i)-f(x_0)).
\end{aligned}
\end{equation*}
By Chebyshev’s inequality, for any fixed $H$,
\begin{equation*}
    \begin{aligned}
        n^{-1}\sum_{i=1}^{n} K_H(X_i-x_0)=\int_{D} K_H(x-x_0)p(x)dx+O_p\left(n^{-1}\sqrt{n\int_{D} K_H^2(x-x_0)p(x)dx}\right).
    \end{aligned}
\end{equation*}
By Theorem~1 in \citep{fan1996study}, this bound holds \emph{uniformly} over $H\in\mathcal{H}_n$ after multiplying the stochastic term by $\sqrt{\log n}$. Hence,
\begin{equation}
\label{eq:int1}
    \begin{aligned}
        &n^{-1}\sum_{i=1}^{n} K_H(X_i-x_0)=\int_{D} K_H(x-x_0)p(x)dx+O_p\left(n^{-1}\sqrt{n\log n\int_{D} K_H^2(x-x_0)p(x)dx}\right)\\
        &=\int_{D_{x_0,H}}K(u)p(x_0+H^{1/2}u)du+o_p(1)=p(x_0)\mu_0^*(x,H) + o_p(1)
    \end{aligned}
\end{equation} uniformly for $H\in\mathcal{H}_n$.

Similarly,
\begin{equation}
\label{eq:int2}
    \begin{aligned}
    &n^{-1}\sum_{i=1}^{n} K_H(X_i-x_0)(f(X_i)-f(x_0))\\
    &=\int_{D} K_H(x-x_0)p(x)(f(x)-f(x_0)) dx+O_p\left(n^{-1}\sqrt{n\log n\int_{D} K_H^2(x-x_0)(f(x)-f(x_0))^2p(x)dx}\right)\\
    &=\int_{D_{x_0,H}}K(u)p(x_0+H^{1/2}u)(f(x_0+H^{1/2}u)-f(x_0))du+o_p(||H||)\\
    &=\int_{D_{x_0,H}}K(u)(p(x_0)+\nabla p(x_0)H^{1/2}u+O(||H||))(\nabla f(x_0)^\top H^{1/2}u+O(||H||))du+o_p(||H||)\\
    &=p(x_0)\nabla f(x_0)^\top H^{1/2}\mu_1^*(x_0,H)+\nabla p(x_0)H^{1/2}\mu_2^*(x_0, H)H^{1/2}\nabla f(x_0)+o_p(||H||)\\
    &= p(x_0)\nabla f(x_0)^\top H^{1/2}\mu_1^*(x_0,H)+O_p(||H||)
    \end{aligned}
\end{equation}
 uniformly for all $H\in\mathcal{H}_n$.

Combining Equations \ref{eq:int1}, \ref{eq:int2}, we have  
 \begin{equation*}
\Bias_\NW(x_0)=\nabla f(x_0)^\top H^{1/2}\bar\mu_1^\star(x_0,H)
\;+\; O_p(\|H\|)
\end{equation*}
uniformly for all $H\in\mathcal{H}_n$. 

Then we calculate variance
\begin{equation*}
    \Var_\NW(x_0)=(\bm 1^\top W \bm 1)^{-1}(\bm 1^\top \Sigma\bm 1)(\bm 1^\top W \bm 1)^{-1}
\end{equation*}
where $\Sigma:=\text{diag}(K_H^2(X_i-x_0)\sigma^2(X_i))$. Analogously to \eqref{eq:int1},
\begin{equation}
\label{eq:int3}
\begin{aligned}
    &n^{-1}(\bm X^\top \Sigma \bm X)=  n^{-1}\sum_{i=1}^{n} K_H^2(X_i-x_0)\sigma^2(X_i)\\
    &=\int_{D} K_H^2(x-x_0)\sigma^2(x)p(x)dx+O_p\left(n^{-1}\sqrt{n\log n\int_{D} K_H^4(x-x_0)\sigma^4(x)p(x)dx}\right)\\
         &=|H|^{-1/2}\int_{D_{x_0,H}} K^2(u)\sigma^2(x_0+H^{1/2}u)p(x_0+H^{1/2}u)du+O_p\left(\sqrt{\frac{\log n}{n}}|H|^{-3/4}\right)\\
         &=|H|^{-1/2}\int_{D_{x_0,H}} K^2(u)\sigma^2(x_0)p(x_0)du(1+o_p(1))+O_p\left(\sqrt{\frac{\log n}{n}}|H|^{-3/4}\right)\\
         &=|H|^{-1/2}\sigma^2(x_0)p(x_0)\int_{D_{x_0},H}K^2(u)du (1+o_p(1))
\end{aligned}
\end{equation}
 uniformly for all $H\in\mathcal{H}_n$.

Combining Equations \ref{eq:int1}, \ref{eq:int3}, we have that
\begin{equation*}
\Var_\NW(x)=\Theta_p(\frac{1}{n|H|^{1/2}})
\end{equation*}
holds uniformly for all $H\in\mathcal{H}_n$. 
\end{proof}

\subsection{Integral Error Estimation of Local Constant Regression}
Here we calculate the integral error for NW estimator.
From Lemma \ref{lem: point bias} we have that for any $x\in D$, the pointwise bias of NW estimator is
\[
\Bias_\NW(x)=\nabla f(x)^\top H^{1/2}\bar\mu_1^\star(x,H)
\;+\; O_p(\|H\|).
\]
By symmetry of the kernel $K$, the first moment $\bar\mu_1^\star(x,H)=0$ if and only if $x\notin\mathcal{B}(H)$. Roughly, $\Bias_\NW(x)=O(\|H^{1/2}\|)$ when $x\in\mathcal{B}(H)$ and $\Bias_\NW(x)=O(\|H\|)$ when $x\notin \mathcal{B}(H)$. Hence it is important to bound $\int_{\mathcal{B}(H)}\MSE_\NW(x)\,dx$ and $\int_{D \setminus \mathcal{B}(H)}\MSE_\NW(x)\,dx$ separately. We first define a smooth substitute for the boundary layer $\mathcal{B}(H)$.

\begin{definition}[Uniform Boundary Layer]
    For any bandwidth $H$ and $\alpha\in(0,1)$, the uniform boundary layer of $D$ is
    \[\mathcal{C}(H,\alpha):=\{x=y-te(y):y\in\partial D, t\in[0,\alpha h_n(y)]\}\] where for any $y\in\partial D$, we define $e(y)$
as the inward unit normal at $y$, and
define
$h_n(y):=\sqrt{\,e(y)^\top H\,e(y)\,}$.
\end{definition}
We now show that one can choose a constant $\alpha\in(0,1)$ (independent of $H$) that “sandwiches’’ $\mathcal{B}(H)$ between $\mathcal{C}(H,\alpha)$ and $\mathcal{C}(H,\alpha^{-1})$.

\begin{lemma}
\label{lem:lb dom}
    Under Assumptions \ref{asp:smooth boundary},  \ref{asp:ball k}, there exists $\alpha\in(0,1)$ such that for all sufficiently small $H$, we have
\[
\mathcal{C}(H,\alpha) \subset\ \mathcal B(H)\subset\mathcal{C}(H,\alpha^{-1}).
\]
\end{lemma}
\begin{proof}
We prove the claim with $\alpha=\min\!\bigl(\tfrac{1}{2}, \tfrac{1}{2}\kappa_1^{-1/2}\bigr)$.

\emph{Step 1: $\mathcal{C}(H,\alpha)\subset \mathcal{B}(H)$.}
Let $\phi$ denote the signed distance to $\partial D$ (positive inside $D$). For $x=y-t e(y)\in\mathcal{C}(H,\alpha)$ and any $v\in\mathbb{R}^d$, the standard expansion (using the shape operator $S_y$) gives
\[
\phi(x+v)= -t + v_n - \tfrac{1}{2}v_T^\top S_y v_T + O(\|v\|^3),
\]
where $v_n:=v^\top e(y)$ and $v_T:=(I-e(y)e(y)^\top)v$. By Assumption \ref{asp:smooth boundary}, $\|S_y\|$ is uniformly bounded.

Define
\[
u:=\kappa_1^{-1/2}\,h_n(y)\,H^{-1/2}e(y)\in\mathbb{R}^d.
\]
Then
\[
\|u\|^2=\kappa_1^{-1}\,h_n(y)^2\,e(y)^\top H^{-1} e(y)
=\kappa_1^{-1}\,\frac{\bigl(e(y)^\top H e(y)\bigr)\bigl(e(y)^\top H^{-1} e(y)\bigr)}{1}
\;\le\; \kappa_1^{-1}\,\kappa(H)\;\le\;1,
\]
since $\kappa(H)\le\kappa_1$ by assumption. Hence $u\in\supp(K)$ (Assumption \ref{asp:ball k}). Using the expansion with $v=H^{1/2}u$ and bounded curvature,
\[
\phi\bigl(x+H^{1/2}u\bigr)= -t + u^\top H^{1/2}e(y) + O(\|H\|^{3/2})
= -t + \kappa_1^{-1/2} h_n(y) + O(\|H\|^{3/2}).
\]
Since $t\le \alpha h_n(y)$ and $\alpha\le \tfrac{1}{2}\kappa_1^{-1/2}$, for sufficiently small $H$ we have $\phi(x+H^{1/2}u)>0$, i.e., $x+H^{1/2}u\notin D$. Thus there exists $u\in\supp(K)$ with $u\notin D_{x,H}$, i.e., $D_{x,H}\neq\supp(K)$, so $x\in\mathcal{B}(H)$.

\emph{Step 2: $\mathcal{B}(H)\subset \mathcal{C}(H,\alpha^{-1})$.}
Let $x=y-t e(y)\in\mathcal{B}(H)$. By definition, there exists $u\in\supp(K)$ with $x+H^{1/2}u\notin D$, i.e.,
\[
\phi(x+H^{1/2}u)= -t + u^\top H^{1/2}e(y) + O(\|H\|) > 0.
\]
Hence
\[
t < u^\top H^{1/2}e(y) + O(\|H\|) \le \|u\|\,\sqrt{e(y)^\top H e(y)} + O(\|H\|) \le h_n(y) + O(\|H\|).
\]
For sufficiently small $H$, this yields $t<\alpha^{-1} h_n(y)$ (since $\alpha^{-1}\ge 2$), so $x\in\mathcal{C}(H,\alpha^{-1})$.
\end{proof}

With $\mathcal{C}(H,\alpha^{-1})$ in hand, we bound the integrated variance and squared bias of the $\NW$ estimator.

\begin{lemma}
    \label{lem:point ub}
    Under Assumptions \ref{asp:smooth boundary}, \ref{asp: smooth func}, and \ref{asp:ball k}, we have, uniformly for all $H\in\mathcal{H}_n$,
    \[
    \int_D\Bias_\NW^2(x)\,dx=O_p\bigl(\|H\|^{3/2}\bigr)
    \quad\text{and}\quad
    \int_D\Var_\NW(x)\,dx=\Theta\!\bigl(n^{-1}|H|^{-1/2}\bigr).
    \]
\end{lemma}
\begin{proof}
    \begin{equation*}
        \begin{aligned}
        &\int_D\Bias_\NW^2(x)dx=\int_{\mathcal{B}(H)}\Bias_\NW^2(x)dx+\int_{D\backslash\mathcal{B}(H)}\Bias_\NW^2(x)dx\\
        &=O_p(||H||)\int_{\mathcal{B}(H)}dx+O_p(||H^2||)\\
        &\leq O_p(||H||)\int_{\mathcal{C}(H,\alpha^{-1})}dx+O_p(||H^2||)\\
        &= O_p(||H||)\int_{\partial D}\int_{0}^{\frac{h_n(y)}{\alpha}}dtdS(y)+O_p(||H^2||)\\
        &=O_p(||H^{3/2}||),
        \end{aligned}
    \end{equation*}
    since $h_n(y)=\sqrt{e(y)^\top H e(y)}=\Theta(\|H\|^{1/2})$ uniformly under bounded condition number.

    For the variance, Lemma \ref{lem: point bias} gives $\Var_\NW(x)=\Theta_p(n^{-1}|H|^{-1/2})$ pointwise (uniformly over $H\in\mathcal{H}_n$). Integrating over $D$ (whose volume is constant) preserves the order:
    \begin{equation*}
        \int_D\Var_\NW(x)dx=\int_D\Theta(n^{-1}|H^{1/2}|^{-1})dx=\Theta(n^{-1}|H^{1/2}|^{-1}).
    \end{equation*}
\end{proof}

\subsection{Proof of Proposition \ref{prop:separation}}
Combining Lemmas \ref{lem:GLR-const-lb}, \ref{lem:point ub}, we obtain the final conclusion.
\begin{theorem}[Precise statement of Proposition \ref{prop:separation}]
Under Assumptions \ref{asp:smooth boundary}, \ref{asp: smooth func}, \ref{asp:ball k}, if the function $f$ to be estimated is not within the the global affine class $\mathcal{G}$,then
\[
\mathbb{E}\!\int_D\MSE_\Gl(x)\,dx=\Omega(1)
\quad\text{and}\quad
\mathbb{E}\!\int_D\MSE_\NW(x)\,dx=O\!\bigl(n^{-3/(d+3)}\bigr),
\]
where $\widehat{f}_\NW$ uses the optimal bandwidth $H\in\mathcal{H}_n$.
\end{theorem}

\begin{proof}
The lower bound for $\MSE_\Gl$ follows directly from Lemma \ref{lem:GLR-const-lb}.

For $\NW$, Lemma \ref{lem:point ub} and the definition of $\mathcal{H}_n$ imply
\[
\int_D\Bias_\NW^2(x)\,dx=O_p(h^3),\qquad
\int_D\Var_\NW(x)\,dx=\Theta\!\bigl(n^{-1}h^{-d}\bigr)
\]
uniformly for $h\in[n^{-a},n^{-b}]$ (since $\|H\|=\Theta(h^2)$ and $|H|^{1/2}=\Theta(h^d)$ under bounded condition number). Hence
\[
\int_D \MSE_\NW(x)\,dx
=O_p\!\Bigl(h^3+\frac{1}{n\,h^{d}}\Bigr)
=O_p\!\bigl(n^{-3/(d+3)}\bigr),
\]
at the minimizer $h=n^{-1/(d+3)}\in[n^{-a},n^{-b}]$ of $h^3+(nh^d)^{-1}$.
\end{proof}

\section{Appendix: Proof of Proposition \ref{prop:lc_and_lla}}\label{appendix:proof2}

The local linear estimator at $x\in D$ is
\[\widehat{f}_{\Ll}(x) = (\bm X^\top W \bm X)^{-1}\bm X^\top W\bm Y,\] where \[\bm X:=\begin{bmatrix}
    1&\ldots&1\\
    X_1-x_0&\ldots&X_n-x_0
\end{bmatrix}^\top .\]

Here we inherit Assumptions \ref{asp:smooth boundary}, \ref{asp: smooth func}, \ref{asp:ball k} and continue to set $d_y=1$ as in Appendix \ref{appendix: proof sepa}.

\subsection{Point-wise Error Estimation}
We have already derived the point-wise error of the local constant estimator in Lemma \ref{lem:point ub}. We now do the same for the local linear estimator.

\begin{lemma}
\label{lem: point bias ll}
Under Assumption \ref{asp: smooth func}, we have, uniformly for all $H\in\mathcal{H}_n$,
\[
\Bias_\Ll(x)=O_p(\|H\|),\qquad
\Var_{\Ll}(x)=O_p\!\Bigl(\frac{1}{n|H|^{1/2}}\Bigr),
\]
where $\mathcal{H}_n:=\{H=h^2B:\ h\in[n^{-a}, n^{-b}],\ B\succ 0,\ |B|=1,\ \kappa(B)\leq \kappa_1\}$ with constants $0<b<a<1$.
\end{lemma}

\begin{proof}
    For any fixed point $x_0\in D$, we have
\begin{equation*}
\begin{aligned}
    \Bias_\Ll(x_0) &= e_1(\bm X^\top W \bm X)^{-1}\bm X^\top W(\bm f-\bm X(f(x_0), \nabla f(x_0)^\top)^\top)\\
    &=\frac{1}{2}e_1(\bm X^\top W \bm X)^{-1}\bm X^\top W(Q+o_p(\text{tr}(H)))
\end{aligned}
\end{equation*}
where $Q=[(X_1-x)^\top\mathcal{H}_f(X_1-x), \ldots, (X_n-x)^\top\mathcal{H}_f(X_n-x)]^\top$.
\begin{equation*}
    \Var_{\Ll}(x_0)=e_1(\bm X^\top W \bm X)^{-1}(\bm X^\top \Sigma\bm X)(\bm X^\top W \bm X)^{-1}e_1^\top
\end{equation*}
where $\Sigma:=\text{diag}(K_H^2(X_i-x_0)\sigma^2(X_i))$.

\begin{equation*}
   n^{-1}(\bm X^\top W \bm X)= \begin{bmatrix}
 n^{-1}\sum_{i=1}^{n} K_H(X_i-x_0)
& n^{-1}\sum_{i=1}^{n} K_H(X_i-x_0)\,(X_i-x_0)^{\top}
\\
 n^{-1}\sum_{i=1}^{n} K_H(X_i-x_0)\,(X_i-x_0)
& n^{-1}\sum_{i=1}^{n} K_H(X_i-x_0)\,(X_i-x_0)(X_i-x_0)^{\top}
\end{bmatrix}
\end{equation*}

\begin{equation*}
   n^{-1}(\bm X^\top \Sigma \bm X)= \begin{bmatrix}
 n^{-1}\sum_{i=1}^{n} K_H^2(X_i-x_0)\sigma^2(X_i)
& n^{-1}\sum_{i=1}^{n} K_H^2(X_i-x_0)\sigma^2(X_i)(X_i-x_0)^{\top}
\\
 n^{-1}\sum_{i=1}^{n} K_H^2(X_i-x_0)\sigma^2(X_i)(X_i-x_0)
& n^{-1}\sum_{i=1}^{n} K_H^2(X_i-x_0)\sigma^2(X_i)(X_i-x_0)(X_i-x_0)^{\top}
\end{bmatrix}
\end{equation*}

\begin{equation*}
    \begin{aligned}
       &n^{-1} \bm X^\top W (\bm f-\bm X(f(x_0), \nabla f(x_0)^\top)^\top)\\
       &=\begin{bmatrix}
            n^{-1}\sum_{i=1}^{n} K_H(X_i-x_0)(f(X_i)-f(x_0)-\nabla f(x_0)^\top (X_i-x_0))\\
            n^{-1}\sum_{i=1}^{n} K_H(X_i-x_0)(f(X_i)-f(x_0)-\nabla f(x_0)^\top (X_i-x_0))(X_i-x_0)
        \end{bmatrix}
    \end{aligned}
\end{equation*}

By the same uniform LLN and $\sqrt{\log n}$ arguments used in the proof of Lemma \ref{lem: point bias}, we have uniformly over $H\in\mathcal{H}_n$:

\begin{equation*}
    \begin{aligned}
         &n^{-1}\sum_{i=1}^{n} K_H^2(X_i-x_0)\sigma^2(X_i)=|H|^{-1/2}\sigma^2(x_0)p(x_0)R_0(x_0, H)(1+o_p(1))\\
         &n^{-1}\sum_{i=1}^{n} K_H^2(X_i-x_0)\sigma^2(X_i)(X_i-x_0)^\top\\
         &\quad=\sigma^2(x_0)p(x_0)|H|^{-1/2}R_1(K)  H^{1/2}+O_p\left((\sqrt{\frac{\log n}{n}}|H|^{-3/4}+1)H^{1/2}\bm 1\right)\\
         &n^{-1}\sum_{i=1}^{n} K_H(X_i-x_0)=p(x_0)\mu_0^*(x_0,H) + o_p(1)\\
         &n^{-1}\sum_{i=1}^{n} K_H^2(X_i-x_0)\sigma^2(X_i)(X_i-x_0)(X_i-x_0)^\top\\
         &\quad =\sigma^2(x_0)p(x_0)|H|^{-1/2}H^{1/2}R_2(K)H^{1/2}+o_p\left(|H|^{-1/2}H\right)\\
             \end{aligned}
\end{equation*}
\begin{equation*}
    \begin{aligned}
         &n^{-1}\sum_{i=1}^{n} K_H(X_i-x_0)(X_i-x_0)^\top=p(x_0)\mu_1^*(x_0,H)^\top H^{1/2}+O_p(H\bm 1)\\
         &n^{-1}\sum_{i=1}^{n} K_H(X_i-x_0)(X_i-x_0)(X_i-x_0)^\top=O_p(H)\\
         &n^{-1}\sum_{i=1}^{n} K_H(X_i-x_0)(f(X_i)-f(x_0)-\nabla f(x_0)^\top (X_i-x_0))=p(x_0)\mu_2(K)\text{tr}\left(H\mathcal{H}_f(x_0)\right)+o_p(\text{tr}(H))\\
         &n^{-1}\sum_{i=1}^{n} K_H(X_i-x_0)(f(X_i)-f(x_0)-\nabla f(x_0)^\top (X_i-x_0))(X_i-x_0)=O_p(H^{3/2}\bm 1),
    \end{aligned}
\end{equation*}
where 
\[
R_0(x_0, H)=\int_{D_{x_0, H}}K^2(u)du,\quad R_1(x_0, H)=\int_{D_{x_0, H}}K^2(u)udu,\quad R_2(x_0, H)=\int_{D_{x_0, H}}K^2(u)uu^\top du.
\]

Then we have the expression of every element in all matrices. A standard blockwise inversion therefore yields
\begin{equation*}
\Bias_\Ll(x)=O_p(\|H\|),\quad \Var_{\Ll}(x)=O_p(\frac{1}{n|H|^{1/2}}).
\end{equation*}
uniformly over $H\in\mathcal{H}_n$, proving the claim.

\end{proof}

\subsection{Integral Error Estimation}
From Lemma \ref{lem: point bias} we have that for any $x\in D$, the pointwise bias of local constant estimator is
\[
\Bias_\NW(x)=\nabla f(x)^\top H^{1/2}\bar\mu_1^\star(x,H)
\;+\; O_p(\|H\|).
\]
By symmetry of the kernel $K$, the first moment satisfies $\bar\mu_1^\star(x,H)=0$ if and only if $x\notin\mathcal{B}(H)$. We will show that if $f$ has a sufficiently large normal gradient on a measurable subset of $\partial D$, then $\int_{\mathcal{B}(H)}\Bias^2_\NW(x)dx$ will have a dominant order $\Theta(\|H^{3/2}\|)$.

We begin by lower-bounding $|\nabla f(x)^\top H^{1/2}\bar\mu_1^\star(x,H)|$, to do which we first lower-bound $e(y)^\top H^{1/2} \bar\mu_1^*(y-te(y), H)$.

\begin{lemma}
\label{lem:lb mom}
    Under Assumption \ref{asp:smooth boundary}, \ref{asp:ball k}, there exists $\alpha\in(0,1)$, $c_*>0$ and $C_*>0$ such that for all sufficiently small $H$ and all $y\in\partial D$, all $t\in[0, \alpha h_n(y)]$,
    \[
    |e(y)^\top H^{1/2} \bar\mu_1^*(y-te(y), H)|\geq c_*h_n(y)\qquad ||\bar\mu_1^*(y-te(y), H)||\leq C_*.
    \]
\end{lemma}
\begin{proof}
Recall $D_{x,H}=\{u\in\mathbb{B}^d:\ x+H^{1/2}u\in D\}$. For $x=y-te(y)$ with $y\in\partial D$, choose an orthogonal $Q(y)$ such that
$Q(y)H^{1/2}e(y)=h_n(y)\,e_d$. Define the rotated domain
\[
\widetilde{D}_{y,t,H}:=\{\,v\in\mathbb{B}^d:\ y-te(y)+H^{1/2}Q(y)v\in D\,\}.
\]
Define the corresponding moments on $\widetilde{D}_{y,t,H}$:
\[
\mu_{0,v}^\star(y,t,H):=\int_{\widetilde{D}_{y,t,H}} K(u)\,du,\quad
\mu_{1,v}^\star(y,t,H):=\int_{\widetilde{D}_{y,t,H}} u\,K(u)\,du,\quad
\bar\mu_{1,v}^\star(y,t,H):=\frac{\mu_{1,v}^\star(y,t,H)}{\mu_{0,v}^\star(y,t,H)}.
\]
Then
\[
\mu_{0,v}^\star(y,t,H)=\mu_0^\star(x,H),\quad
Q(y)\,\mu_{1,v}^\star(y,t,H)=\mu_1^\star(x,H),\quad
Q(y)\,\bar\mu_{1,v}^\star(y,t,H)=\bar\mu_1^\star(x,H),
\]
and hence
\begin{equation}
\label{eq:trans int}
e(y)^\top H^{1/2}\bar\mu_1^\star(x,H)=h_n(y)\,e_d^\top \bar\mu_{1,v}^\star(y,t,H).
\end{equation}

By the standard signed-distance expansion with shape operator $S_y$ (Assumption \ref{asp:smooth boundary}), for $v\in\mathbb{R}^d$,
\[
\phi\bigl(y-te(y)+H^{1/2}Q(y)v\bigr)
= -t + h_n(y)\,v_d - \tfrac12 z_T^\top S_y z_T + O(\|H\|^{3/2}),
\]
where $z_T:=(I-e(y)e(y)^\top)H^{1/2}Q(y)v$ and $\|S_y\|$ is uniformly bounded. Therefore
\[
y-te(y)+H^{1/2}Q(y)v\in D\ \iff\ v_d<\frac{t}{h_n(y)}+O(\|H\|^{1/2}),
\]
so
\[
\widetilde{D}_{y,t,H}=\Bigl\{\,v\in\mathbb{B}^d:\ v_d< \frac{t}{h_n(y)}+O(\|H\|^{1/2})\,\Bigr\}.
\]

Consequently,
\[
\begin{aligned}
\mu_{0,v}^\star(y,t,H)
&=\int_{\widetilde{D}_{y,t,H}} K(v)\,dv\\
&=\int_{\mathbb{R}^{d-1}}\!\!dv_{1:d-1}\int_{-\sqrt{1-\sum_{i=1}^{d-1}v_i^2}}^{t/h_n(y)} K(v_{1:d-1},v_d)\,dv_d
\;+\;O(\|H\|^{1/2})\\
&\ge \int_{\mathbb{R}^{d-1}}\!\!dv_{1:d-1}\int_{-\sqrt{1-\sum_{i=1}^{d-1}v_i^2}}^{0} K(v_{1:d-1},v_d)\,dv_d
\;+\;O(\|H\|^{1/2}).
\end{aligned}
\]
Using positivity and boundedness of $K$ on $\mathbb{B}^d$, there exist $0<C<D<\infty$ such that
\begin{equation}
\label{eq:bd1}
C<\mu_{0,v}^\star(y,t,H)<D .
\end{equation}
Similarly,
\[
\begin{aligned}
e_d^\top \mu_{1,v}^\star(y,t,H)
&=\int_{\widetilde{D}_{y,t,H}} v_d K(v)\,dv\\
&=\int_{\mathbb{R}^{d-1}}\!\!dv_{1:d-1}\int_{-\sqrt{1-\sum_{i=1}^{d-1}v_i^2}}^{t/h_n(y)} v_d\,K(v_{1:d-1},v_d)\,dv_d
\;+\;O(\|H\|^{1/2})\\
&=:g\!\left(\tfrac{t}{h_n(y)}\right)+O(\|H\|^{1/2}).
\end{aligned}
\]
Note $g(0)<0$ and $g(\tau)$ decreases as $\tau\downarrow 0$. Choose $\alpha=\alpha_1\in(0,1)$ with $g(\alpha_1)<0$. Then for all $t\in[0,\alpha_1 h_n(y)]$,
\begin{equation}
\label{eq:bd2}
e_d^\top \mu_{1,v}^\star(y,t,H)\le g(\alpha_1)+O(\|H\|^{1/2})<0 .
\end{equation}
Combining \eqref{eq:bd1}–\eqref{eq:bd2}, for small $H$,
\[
\bigl|e_d^\top \bar\mu_{1,v}^\star(y,t,H)\bigr|
=\frac{|e_d^\top \mu_{1,v}^\star(y,t,H)|}{\mu_{0,v}^\star(y,t,H)}
\;\ge\; \frac{|g(\alpha_1)|}{2D}.
\]
Using \eqref{eq:trans int} yields the first bound with $c_*:=|g(\alpha_1)|/(2D)$. The second bound follows from boundedness of $K$ and $\supp(K)\subset\mathbb{B}^d$.
\end{proof}

Since the leading term of $\Bias_\NW(x)$ is $\nabla f(x)^\top H^{1/2}\bar\mu_1^\star(x,H)$, the lower bound on $|e(y)^\top H^{1/2}\bar\mu_1^\star(y-te(y),H)|$ alone does not guarantee order $\Theta(\|H^{1/2}\|)$; we also require a sufficiently large normal derivative of $f$ along $\partial D$.
\begin{definition}[Extreme boundary gradient class]For any domain $D$ and constants $m$ and $M$, we define a class of functions $\mathcal{E}(D, m, M)$, where $f\in\mathcal{E}(D, m, M)$ iff there exist a measurable $\Gamma\subset\partial D$ with $S(\Gamma)>0$ and constants $m$, $M$ such that $|\partial_e f(y)|\geq m$ and $||\nabla_Tf(y)||<M$ where $\partial_ef(y)=\nabla f(y)^\top e(y)$ and $\nabla_Tf(y)=(I-e(y)e(y)^\top)\nabla f(y)$.
\end{definition}

Then we can prove that if the function to be estimated is within $\mathcal{E}(D, m, M)$ where $m$ and $M$ are specifically chosen constants that are independent of $H$, the NW has an integral squared bias with high order.

\begin{lemma}
\label{lem: lb omega_p}
Under Assumptions \ref{asp:smooth boundary}, \ref{asp: smooth func}, and \ref{asp:ball k}, if $f\in\mathcal{E}(D,m,M)$ with
\[
c_*^2\,\kappa_1^{-1}m^2-2c_*\,\kappa_1^{-1/2}C_*\,mM\;\ge\; C_1>0,
\]
then uniformly over $H\in\mathcal{H}_n$,
\[
\int_D\Bias_\NW^2(x)\,dx=\Omega_p(\|H\|^{3/2})
\qquad\text{and}\qquad
\int_D\Var_\NW(x)\,dx=\Omega_p\!\bigl(n^{-1}|H|^{-1/2}\bigr).
\]
\end{lemma}
\begin{remark}
It is easy to construct $f$ and $D$ satisfying Assumptions \ref{asp:smooth boundary}, \ref{asp: smooth func} and $f\in\mathcal{E}(D,m,M)$. For example, on $D=\{(x_1,x_2):x_1^2+x_2^2\le1\}$, $f(x_1,x_2)=\frac{\sqrt{c_1\kappa_1}}{2c_*}(x_1^2+x_2^2)$ works with suitable $m,M$.
\end{remark}

\begin{proof}
    We first lower-bound bias. Choosing $\alpha$ as the minimum $\alpha$ given by Lemmas \ref{lem:lb dom}, \ref{lem:lb mom}, we have
    \begin{equation}
    \label{eq: bias int}
    \begin{aligned}
        &\int\Bias_\NW^2(x)dx\geq\int_{\mathcal{B}(H)}\Bias_\NW^2(x)dx\geq\int_{\mathcal{C}(H,\alpha)}\Bias_\NW^2(x)dx\\
        &=\int_{\partial D}\int_{0}^{\alpha h_n(y)}\Bias_\NW^2(y-te(y))\text{det}\left(I-tS_y\right)dtdS(y)\\
        &\geq C_2\int_{\Gamma}\int_0^{\alpha h_n(y)}\left( (\nabla f(y-te(y))^\top H^{1/2}\bar\mu_1^\star(y-te(y),H))^2
\;+\; O_p(||H^{3/2}||)\right)dtdS(y),
    \end{aligned}
\end{equation}
where the last inequality is derived from Lemma \ref{lem: point bias} and the boundedness of $S_y$ from Assumption \ref{asp:smooth boundary}.

For all $t\in\alpha h_n(y)$ and all $\eta\in(0,1)$,
\begin{equation*}
    \begin{aligned}
        &\left(\nabla f(y-te(y))^\top H^{1/2}\bar\mu_1^\star(y-te(y),H)\right)^2\\
        &\overset{(a)}=\left(\partial_e f(y-te(y))e(y)^\top H^{1/2}\bar\mu_1^\star(y-te(y),H)+\nabla_T f(y-te(y))^\top H^{1/2}\bar\mu_1^\star(y-te(y),H)\right)^2\\
        &\overset{(b)}\geq (1-\eta)\left(\partial_e f(y-te(y))e(y)^\top H^{1/2}\bar\mu_1^\star(y-te(y),H)\right)^2-\eta^{-1}\left(\nabla_T f(y-te(y))^\top H^{1/2}\bar\mu_1^\star(y-te(y),H)\right)^2\\
        &\overset{(c)}\geq (1-\eta)(\partial_ef(y-te(y)))^2h_n^2(y)c_*^2-\eta^{-1}||\nabla_T f(y-te(y)||^2C_*^2||H||\\
        &=(1-\eta)(\partial_ef(y))^2h_n^2(y)c_*^2-\eta^{-1}||\nabla_T f(y)||^2C_*^2||H|| + O(||H^{3/2}||)\\
        &\geq \left((1-\eta)m^2c_*^2\kappa_1^{-1}-\eta^{-1}M^2C_*^2\right)||H||+O(||H^{3/2}||).
    \end{aligned}
\end{equation*}
Here $(a)$ uses the fact $\nabla f=\partial_e f(y)e(y)+\nabla_T f(y)$, $(b)$ uses the fact $(A+B)^2\geq (1-\eta)A^2-\eta^{-1}B^2$ for all $\eta\in(0,1)$, $(c)$ uses Lemma \ref{lem:lb mom}.

Setting $\eta=\frac{MC_*\kappa_1^{1/2}}{mc_*}$, we have 
\begin{equation}
\label{eq: bias item}
    \begin{aligned}
        &\left(\nabla f(y-te(y))^\top H^{1/2}\bar\mu_1^\star(y-te(y),H)\right)^2\\
        &\geq\left((1-\eta)m^2c_*^2\kappa_1^{-1}-\eta^{-1}M^2C_*^2\right)||H||+O(||H^{3/2}||)\\
        &= \left(c_*^2\kappa_1^{-1}m^2-2c_*\kappa_1^{-1/2}C_*mM\right)||H||+O(||H^{3/2}||)\\
        &\geq C_1||H||++O(||H^{3/2}||).
    \end{aligned}
\end{equation}
Combining Equations \ref{eq: bias int}, \ref{eq: bias item},  we have that 

\begin{equation*}
    \begin{aligned}
        \int\Bias_\NW^2(x)dx&\geq \alpha C_2\left(C_1||H||
\;+\; O_p(||H^{3/2}||)\right)\int_{\Gamma}h_n(y)dS(y)=\Omega_p(||H^{3/2}||).
    \end{aligned}
\end{equation*}
The last equation holds because $h_n(y)\geq\kappa_1^{-1}||H^{1/2}||$.
Using Lemma \ref{lem: point bias}, it is straightforward to show that \[\int_\Omega\Var_\NW(x)dx=\Omega_p(n^{-1}|H^{1/2}|^{-1}).\]
\end{proof}

We then show that local linear regression enjoys strictly lower bias than NW on all functions.

\begin{lemma}
\label{lem: lb omega_p ll}
    Under Assumptions \ref{asp:smooth boundary}, \ref{asp: smooth func}, \ref{asp:ball k}, we have $\int_D\Bias_\Ll^2(x)dx=O_p(||H^{2}||)$ and $ \int_D\Var_\Ll(x)dx=O_p(n^{-1}|H^{1/2}|^{-1})$ uniformly hold for all $H\in\mathcal{H}_n$.
\end{lemma}
\begin{proof}
    It is straightforward from Lemma \ref{lem: point bias ll}.
\end{proof}

Then it is easy to prove the final conclusion.
\begin{theorem}[Precise statement of Proposition \ref{prop:lc_and_lla}]Under Assumptions \ref{asp:smooth boundary}, \ref{asp: smooth func}, \ref{asp:ball k}, if the function $f$ to be estimated is within $\mathcal{E}(D, m, M)$ where \[c_*^2\kappa_1^{-1}m^2-2c_*\kappa_1^{-1/2}C_*mM\geq C_1>0,\] we have $\mathbb{E}\int_D\MSE_\NW(x)dx=\Omega(n^{-3/(d+3)})$ and $\mathbb{E}\int_D\MSE_\Ll(x)dx=O(n^{-4/(d+4)})$, where both $\widehat{f}_\NW$ and $\widehat{f}_\Ll$ are at their optimal bandwidth $H\in\mathcal{H}_n$.
\end{theorem}
\begin{proof}
    According to Lemma \ref{lem: lb omega_p} and the definition of $\mathcal{H}_n$, we have that we have $\int_D\Bias_\NW^2(x)dx=\Omega_p(h^3)$ and $ \int_D\Var_\NW(x)dx=\Omega_p(n^{-1}h^{-d})$ uniformly hold for all $h\in[n^{-a}, n^{-b}]$. So we have \[\int_D \MSE_\NW(x)dx=\Omega_p(h^3+\frac{1}{nh^d})=\Omega_p(n^{-3/(d+3)})\]even for optimal $H\in\mathcal{H}_n$.
    The last equality holds because $h=n^{-1/(d+3)}\in[n^{-a}, n^{-b}]$ is the minimizer of $h^3+(nh^d)^{-1}$.

    According to Lemma \ref{lem: lb omega_p ll}, we have that \[
    \int_D \MSE_\Ll(x)dx=O_p(h^4+\frac{1}{nh^d})=O_p(n^{-4/(d+4)})
    \]
    for the optimal $H\in\mathcal{H}_n$.
    According to the definition of $\Omega_p$ and $O_p$, it is straightforward to deduce the conclusion. 
\end{proof}

\section{Appendix: Conjugate Gradient Solver}\label{appendix:cg}
The Conjugate Gradient (CG) method \citep{hestenes1952methods} is an iterative algorithm for solving systems of linear equations with symmetric positive-definite matrices.
As described in Section \ref{section:algorithm}, we solve the linear systems with matrix \(\Sigma_i\) for blocks of queries in parallel:
\begin{align}
    \Sigma_i x_i = y_i,\quad \text{for } i\in\{(r-1)B_r+1,\ldots, rB_r\}
\end{align}
In the {FlashLLA} forward algorithm \ref{alg:flashlla_fwd}, CG is applied within each row block, with \(Q_r,M_r,m_r,\omega_r\) being the block quantities, \(X_r\) being the solution to be computed and \(Y_r\) being the right-hand side.
We use the simplest initialization \(X_r\leftarrow 0\) for CG. Hence the initial residual \(r_i\) is set to be \(y_i\), i.e., \(R^{(0)}\leftarrow Y_r\) in Algorithm~\ref{alg:cgsolver}.

The core computation is the matrix-vector product \(\Sigma_i p_i\) for the search vectors \(p_i\) computed in the lines~\ref{ln:subspaceiterstart}-\ref{ln:subspaceiterend}. The result is stored in matrix \(\Sigma_P\).
This operation has high I/O intensity due to the requirement to stream through the entire \(K\) matrix in HBM during each CG iteration.
Consequently, controlling the number of iterations is crucial for both the efficiency and convergence.
While maximal iteration number \(T\le d\) can be manually set, further considerations are necessary to ensure the numerical stability and performance.

First, the convergence and convergent rate of CG are greatly influenced by its spectral condition.
However, the conditioning varies significantly across positions. For example, for early tokens, the matrix \(\Sigma_i\) is low-rank and requires relatively large \(\lambda\) to maintain positive definiteness.
To address this, we make the regularization \(\lambda\) learnable and data-dependent:
\begin{align}
    \lambda_i = \mathrm{sigmoid}(W_\lambda x_i).
\end{align}
The dimension of the weight \(W_\lambda \in \mathbb R^{d\times d_\lambda}\) controls the granularity of the regularization. Setting \(d_\lambda=d\) enables per-dimensional regularization, 
though empirically set \(d_\lambda = d_h\) suffices, where \(d_h\) denotes the head dimension.

Additionally, since the CG solves multiple systems in parallel, different system converge at different iterations.
To prevent the numerical issues from affecting early converged system, we employ an active mask that disables iterations for systems whose residual norm fall below the tolerance \(\epsilon\).

\begin{algorithm}
\caption{{FlashLLA} CG Solver}
\begin{algorithmic}[1]\label{alg:cgsolver}
\REQUIRE{Variables \(K\) in HBM, \(X_r,Y_r,Q_r,M_r,m_r,\omega_r\) in SRAM, block sizes \(B_c\), regularization \(\lambda\), tolerance \(\epsilon\), max iterations \(T\), bandwidth \(h\).}
\STATE Initialize on-chip: \(R^{(0)} \leftarrow Y_r \in \mathbb R^{B_r\times d}, P^{(0)} \leftarrow Y_r \in \mathbb R^{B_r\times d}\).
\FOR{$t = 1$ to $T$}
    \STATE Initialize on-chip: \(\Sigma_P^{(0)} \leftarrow 0 \in \mathbb R^{B_r\times d}\).
    \FOR{$c = 1$ to $\lceil n / B_c \rceil$}\label{ln:subspaceiterstart}
        \STATE Load \(K_c\) from HBM to SRAM.
        \STATE Compute \(W = \exp(Q_r K_c^\top/h-\texttt{bcast}(m_r))\).
        \STATE Compute \(\Sigma_P^{(c)}=\Sigma_P^{(c-1)} + (W \odot P^{(t-1)}K_c)K_c\).
    \ENDFOR
    \STATE Compute \(P_Q = \texttt{brsum}(Y_r\odot Q_r)\) and \(P_M = \texttt{brsum}(P^{(t-1)}\odot M_r)\).
    \STATE Compute \(\Sigma_P^{(\texttt{last})}=\Sigma_P^{(\texttt{last})}-P_Q\odot M_r - P_M\odot Q_r+\texttt{bcast}(\omega_r)\odot P_Q\odot Q_r + \lambda Y_r\)\label{ln:subspaceiterend}
    \STATE Compute \(n = \texttt{rsum}(R^{(t-1)}\odot R^{(t-1)})\) and check convergence.
    \STATE Compute \(\alpha = {n}/{\texttt{rsum}(P^{(t-1)}\odot \Sigma_P^{(\texttt{last})})}\).
    \STATE Compute \(X^{(t)} = X^{(t-1)} + \texttt{bcast}(\alpha)\odot P^{(t-1)}\).
    \STATE Compute \(R^{(t)} = R^{(t-1)} - \texttt{bcast}(\alpha)\odot \Sigma_P^{(\texttt{last})}\).
    \STATE Compute \(\beta = {\texttt{rsum}(R^{(t)}\odot R^{(t)})}/{n}\).
    \STATE Compute \(P^{(t)} = R^{(t)} + \beta \odot P^{(t-1)}\).
\ENDFOR
\RETURN{\(X_r\) as the solution of \(\Sigma_i X = P_i\) for \(i\) in the block.}
\end{algorithmic}
\end{algorithm}

\section{Appendix: Backward Derivation}\label{appendix:backward}
This section provides the detailed derivation of the backward pass.
Defining the following variables:
\begin{align}
\hat b_i &= \sum_{j=1}^i s_{ij}v_j = \sum_{j=1}^i w_{ij}\frac{n_{ij}}{\delta_i} v_j,\quad g_i = \frac{\partial \mathcal L}{\partial \hat b_i} \in \mathbb R^d\\
\gamma_{ij} &= g_i^\top v_j = \frac{\partial \mathcal L}{\partial s_{ij}},\quad \beta_i = \frac{1}{\delta_i}\sum_{j=1}^i \gamma_{ij}{s_{ij}},\quad c_{ij} = \frac{\gamma_{ij}w_{ij}}{\delta_i}
\end{align}
The gradient of the loss with respect to \(v_j\) is given by:
\begin{align}
    \frac{\partial \mathcal L}{\partial v_j} &= \sum_{i=j}^n s_{ij} g_i,\quad \Delta_V = S^\top G
\end{align}
The gradient of \(q_i\) and \(k_j\) is related to \(s_{ij}\), which can be broken down into \(w_{ij}\) and \(z_{ij}\) path in the computation graph.
The partial gradients of the loss with respect to \(w_{ij}, z_{ij}\) are given by:
\begin{align}
    \frac{\partial\mathcal L}{\partial w_{ij}} &= \frac{\partial \mathcal L}{\partial s_{ij}}\frac{\partial s_{ij}}{\partial w_{ij}} + \frac{\partial \mathcal L}{\partial \omega_i}\frac{\partial \omega_i}{\partial w_{ij}} + \frac{\partial \mathcal L}{\partial \mu_i}^\top\frac{\partial \mu_i}{\partial w_{ij}} + \operatorname{Tr}\left(\frac{\partial \mathcal L}{\partial \Sigma_i}^\top\frac{\partial \Sigma_i}{\partial w_{ij}}\right)\\
    &= \frac{\gamma_{ij}n_{ij}}{\delta_i} - \beta_i + z_{ij}^\top \frac{\partial\mathcal L}{\partial \mu_i} + z_{ij}^\top\frac{\partial \mathcal L}{\partial \Sigma_i} z_{ij}\\
    \frac{\partial\mathcal L}{\partial z_{ij}} &= \frac{\partial \mathcal L}{\partial s_{ij}}\frac{\partial s_{ij}}{\partial z_{ij}} + \frac{\partial \mathcal L}{\partial \mu_i}^\top\frac{\partial \mu_i}{\partial z_{ij}} + \operatorname{Tr}\left(\frac{\partial \mathcal L}{\partial \Sigma_i}^\top\frac{\partial \Sigma_i}{\partial w_{ij}}\right)\\
    &= -c_{ij}\rho_{i} + w_{ij}\frac{\partial\mathcal L}{\partial \mu_i} + 2w_{ij}\frac{\partial \mathcal L}{\partial \Sigma_i}z_{ij}
\end{align}
Then the partial gradients of the loss with respect to \(k_j\) and \(q_i\) are given by:
\begin{align}
    \frac{\partial\mathcal L}{\partial k_j} &= \sum_{i=j}^n \frac{\partial\mathcal L}{\partial w_{ij}}\frac{\partial w_{ij}}{\partial k_j} + \frac{\partial\mathcal L}{\partial z_{ij}}\frac{\partial z_{ij}}{\partial k_j} = \sum_{i=j}^n \frac{\partial \mathcal L}{\partial w_{ij}}\frac{w_{ij}}{h}q_i + \frac{\partial \mathcal L}{\partial z_{ij}}\\
    \frac{\partial\mathcal L}{\partial q_i} &= \sum_{j=1}^i \frac{\partial\mathcal L}{\partial w_{ij}}\frac{\partial w_{ij}}{\partial q_i} + \frac{\partial\mathcal L}{\partial z_{ij}}\frac{\partial z_{ij}}{\partial q_i} = \sum_{j=1}^i \frac{\partial \mathcal L}{\partial w_{ij}}\frac{w_{ij}}{h}k_j - \frac{\partial \mathcal L}{\partial z_{ij}}
\end{align}
In order to compute the partial gradients of \(\mu_i\) and \(\Sigma_i\), denote
\begin{align}
    u_i = \Sigma_i^{-1} \sum_{j=1}^i c_{ij}z_{ij} =
    \Sigma_i^{-1}\biggl[\sum_{j=1}^i c_{ij}k_{j} -  \Bigl(\sum_{j=1}^i c_{ij}\Bigr)q_i\biggr]
\end{align}
which can be computed with existing Conjugate Gradient solver~\ref{alg:cgsolver}. Then
\begin{align}
\frac{\partial \mathcal L}{\partial \mu_i} &= \sum_{j=1}^i \frac{\partial \mathcal L}{\partial n_{ij}}\frac{\partial n_{ij}}{\partial \mu_i} + \frac{\partial \mathcal L}{\partial \delta_i}\frac{\partial \delta_i}{\partial \mu_i} = -u_i + 2\beta_i\rho_i\\
\frac{\partial \mathcal L}{\partial \Sigma_i} &= \sum_{j=1}^i \frac{\partial \mathcal L}{\partial n_{ij}}\frac{\partial n_{ij}}{\partial \Sigma_i} + \frac{\partial \mathcal L}{\partial \delta_i}\frac{\partial \delta_i}{\partial \Sigma_i} = -\frac{1}{2}\rho_i\frac{\partial\mathcal L}{\partial \mu_i}^\top + \frac{1}{2}u_i\rho_i^\top
\end{align}
We denote the following variables:
\begin{align}
    \Delta_\mu = -U + 2\texttt{bcast}(\beta)\odot R
\end{align}
We can materialize the gradient of \(w_{ij}\) for every \(i,j\) pair and then perform the reduction.
\begin{align}
    w_{ij}\frac{\partial \mathcal L}{\partial w_{ij}} = \gamma_{ij}s_{ij} + w_{ij}z_{ij}^{\top}\frac{\partial \mathcal L}{\partial \mu_i} + w_{ij}z_{ij}^\top\frac{\partial \mathcal L}{\partial \Sigma_i}z_{ij}
\end{align}
We omit the \(Q,K\) in the \texttt{relmm} for simplicity, then the gradient of \(w_{ij}\) can be computed as:
\begin{align}
    \Delta_W &= \Gamma\odot S + W\odot(-\texttt{bcast}(\beta) + \texttt{relmm}(\Delta_\mu) - \frac{1}{2}\texttt{relmm}(\Delta_\mu) \odot \texttt{relmm}(R)\\& +\frac{1}{2}\texttt{relmm}(U) \odot \texttt{relmm}(R))
\end{align}Then we can compute the gradient of \(k_j\) and \(q_i\) through \(w_{ij}\) branch as follows:
\begin{align}
    \Delta_W^K &= \frac{1}{h}\Delta_W^\top Q,\quad \Delta_W^Q = \frac{1}{h}\Delta_W K
\end{align}
For the \(z_{ij}\) path, we avoid materializing the third-order tensor by performing the reduction internally,
\begin{align}
    \sum_{i=j}^n \frac{\partial \mathcal L}{\partial z_{ij}} &= \sum_{i=j}^n -c_{ij}\rho_i + \sum_{i=j}^n w_{ij}\frac{\partial \mathcal L}{\partial \mu_i} + 2\sum_{i=j}^n w_{ij}\frac{\partial \mathcal L}{\partial \Sigma_i}z_{ij}\\
    \sum_{j=1}^i \frac{\partial \mathcal L}{\partial z_{ij}} &= \sum_{j=1}^i -c_{ij}\rho_i + \sum_{j=1}^i w_{ij}\frac{\partial \mathcal L}{\partial \mu_i} + 2\sum_{j=1}^i w_{ij}\frac{\partial \mathcal L}{\partial \Sigma_i}z_{ij}
\end{align}
Then we can compute the gradient of the loss with respect to \(z_{ij}\) as follows:
\begin{align}
    \Delta_Z^K &= -C^\top R + W^\top \Delta_\mu - (W\odot \texttt{relmm}(\Delta_\mu))^\top R + (W\odot \texttt{relmm}(R))^\top U\\
    \Delta_Z^Q &= -\texttt{brsum}(C)\odot R + \texttt{brsum}(W)\odot \Delta_\mu\\
    &\quad - \texttt{brsum}(W\odot\texttt{relmm}(\Delta_\mu)) \odot R + \texttt{brsum}(W\odot\texttt{relmm}(R)) \odot U
\end{align}Hence the gradient of the loss with respect to \(k_j\) and \(q_i\) can be computed as:
\begin{align}
    \Delta_K &= \Delta_W^K + \Delta_Z^K,\quad \Delta_Q = \Delta_W^Q + \Delta_Z^Q
\end{align}

\begin{figure}[ht]
    \centering
    \includegraphics[width=1\textwidth]{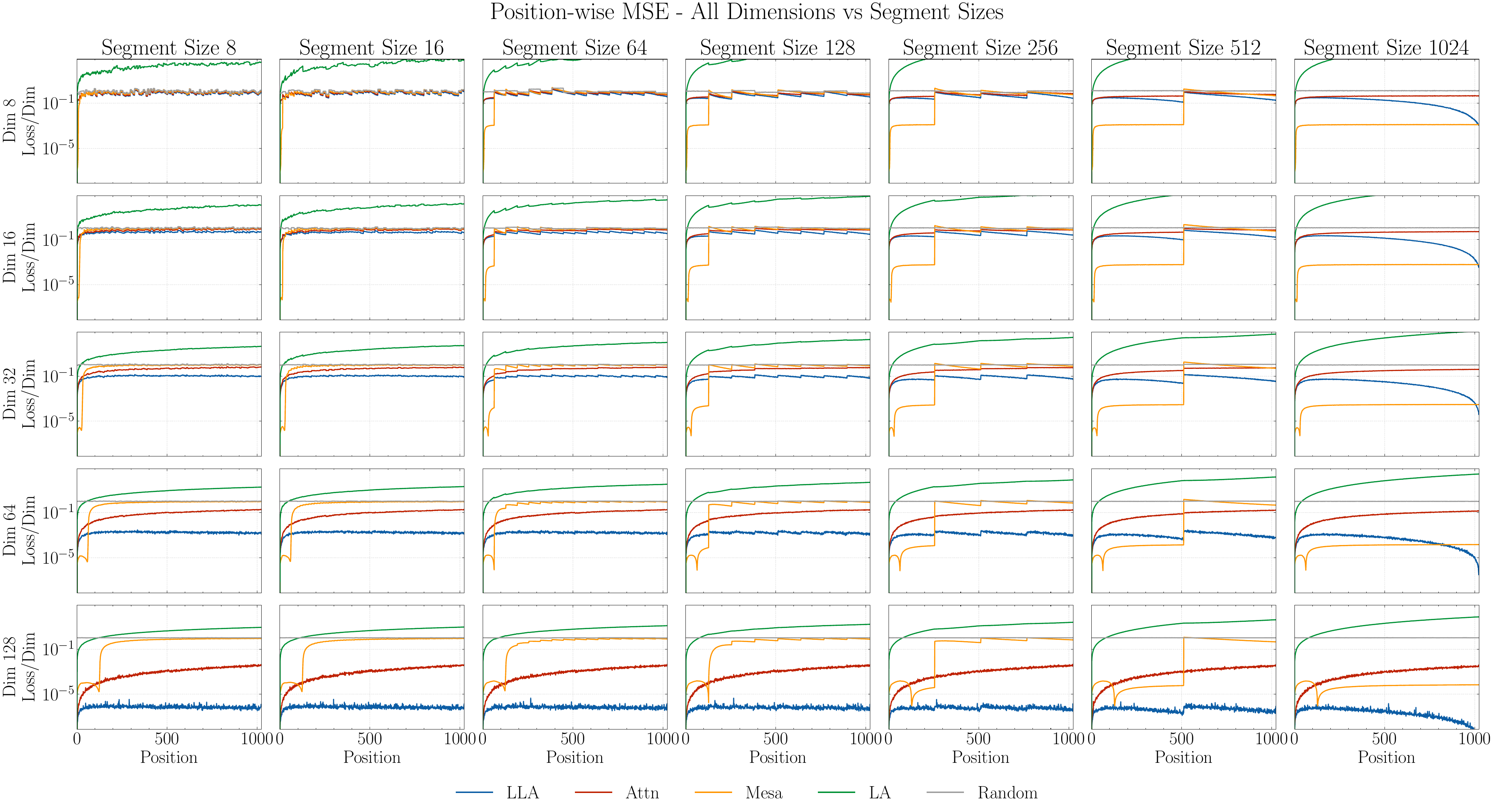}
    \caption{
        Test-time regression with across all input dimension \(d\) and segment size \(S\).
    }
    \label{fig:exp_all1}
\end{figure}
\section{Appendix: Experiments}\label{appendix:exp}
\subsection{Piecewise Linear Data Generation.}
Let \(n=2^m\) be the number of segment with \(n=\log_2 n \le d\). For each section index \(c\in\{1,\ldots, n\}\), define a sign pattern
\(S_c = (s_{c,1}, s_{c,2}, \ldots, s_{c,m}) \in \{-1,+1\}^m\)
by reading the least-significant bits of \(c\).
For each segment in each sample, draw \(Z\sim\mathcal N(0, I_d)\) and construct the data \(X\in\mathbb R^d\) by flipping the first \(m\) coordinate of \(Z\):
\begin{align}
    X_j = \begin{cases}
        S_{c,j}|Z_j|,&j\le m\\
        Z_j,&j > m
    \end{cases}
\end{align}
As the result, the constructed segment is a truncated Gaussian conditioned to lie in the cone \(\mathcal C_c = \{x\in\mathbb R^d: s_{c,j}x_j \ge 0, j=1,\ldots, m\}\) where \(\mathcal C_i\cap\mathcal C_j = \emptyset\).
Denote \(T_c(Z) = (S_c\odot |Z_{1:m}|,|Z_{m+1:d}|)\) and segment distribution \(P_c\), we have
\begin{align}
X\sim P_c \iff X \overset{d}{=} T_c(Z)
\end{align}

\subsection{Training Configuration}
\paragraph{Test-Time Regression.}
This experiment evaluates test-time adaptation without training any model parameters. We sweep over input dimension
$d \in \{8,16,32,64,128\}$ and segment size
$S \in \{8,16,32,64,128,256,512,1024\}$ with fixed sequence length
$L=1024$. Performance is evaluated by averaging the mean squared error over
$10{,}000$ independently generated sequences.

\paragraph{In-Context Regression.}
We fix the input and output dimensions at
$d_x = d_y = 32$. For each random seed, we generate
$100{,}000$ training examples with noise level $\delta=0.1$ and
$1{,}000$ test examples with $\delta=0$ (noiseless evaluation). For each sequence length
$L \in \{64,128,256,512\}$, we sweep over segment size
$S \in \{L/8,L/4,L/2,L\}$ and learning rate
$\{5\times10^{-5},10^{-4},5\times10^{-4},10^{-3}\}$.
All models are trained with the AdamW optimizer
$(\beta_1,\beta_2)=(0.9,0.999)$, weight decay $0.1$, batch size
$256$, for a maximum of $100$ epochs.

\paragraph{In-Context Associative Recall.}
We fix the vocabulary size $|A_k \cup A_v| = 8{,}192$.
For each sequence length $L \in \{64,128,256,512\}$,
we sweep over the number of key-value pairs
$\{L/16,L/8,L/4\}$ and learning rate
$\{10^{-4},5\times10^{-4},10^{-3}\}$.
We generate $20{,}000$, $40{,}000$, and $60{,}000$ training examples
and $1{,}000$ test examples each for the respective key-value pair counts.
Training uses AdamW with $(\beta_1,\beta_2)=(0.9,0.999)$,
weight decay $0.1$, batch size $256$, for a maximum of $32$ epochs.
Short convolution and feature map are disabled for all models.

\paragraph{Permutation State Tracking.}
We fix the vocabulary size $|A| = 8{,}192$.
For each random seed, we generate $100{,}000$ training examples and
$1{,}000$ test examples. For each position count
$N \in \{16,24,48,96\}$, we sample the number of instructions
$S \sim \text{Uniform}(N/6,N/3)$ and use $8$ queries per example.
We sweep over learning rate $\{10^{-4},5\times10^{-4},10^{-3}\}$.
Models are trained with AdamW $(\beta_1,\beta_2)=(0.9,0.999)$,
weight decay $0.1$, batch size $256$, for a maximum of $64$ epochs.

\subsection{Additional Experiment Results}\label{appendix:add_results}
\begin{figure}[ht]
    \centering
    \includegraphics[width=0.8\textwidth]{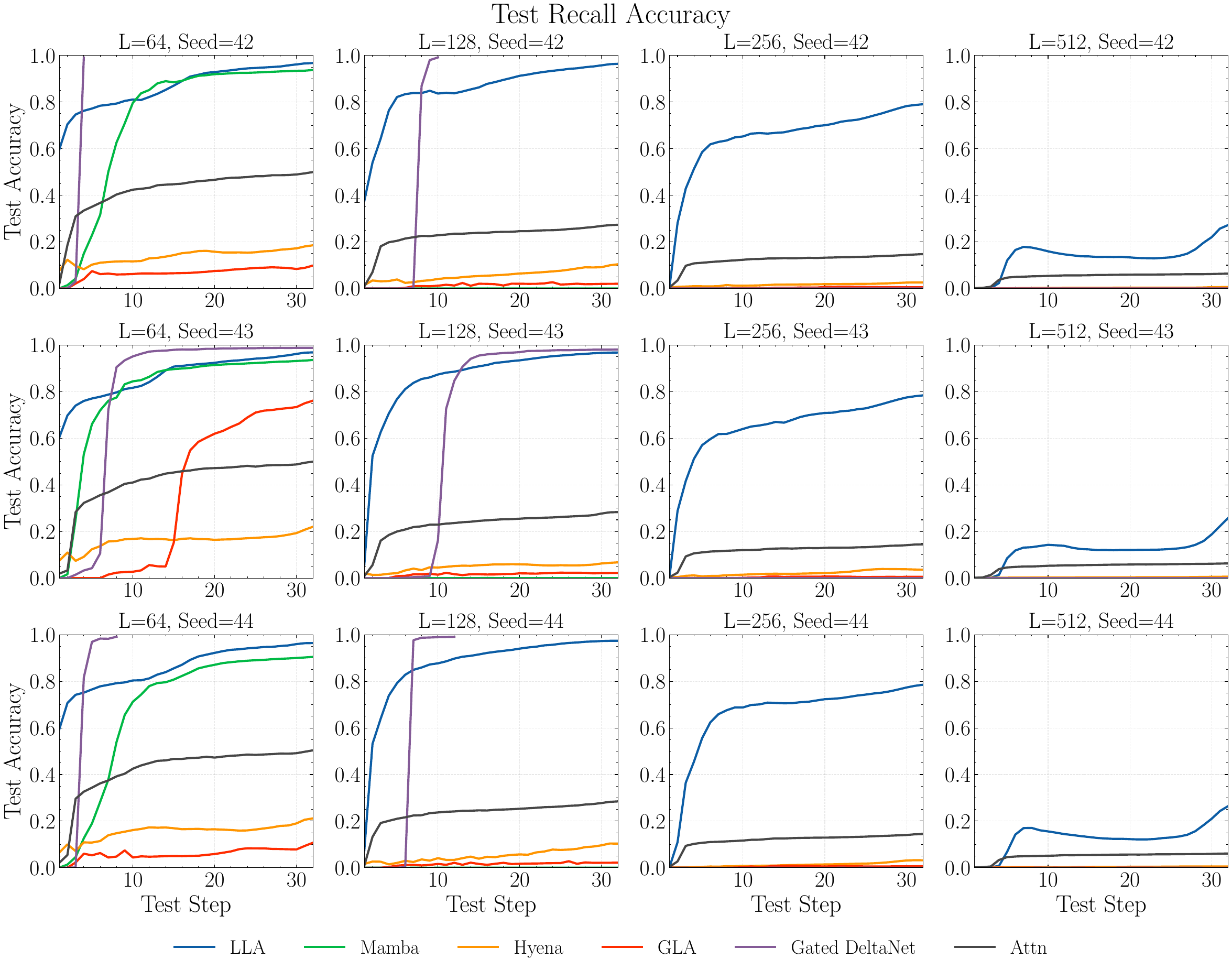}
    \caption{
        Test accuracy curves for associative recall across all sequence lengths \(L\) (averaged over 3 random seeds). Results for different numbers of key-value pairs are averaged within each sequence length for visual clarity
    }
    \label{fig:exp_all2}
\end{figure}
In Figure~\ref{fig:exp_all1}, we provide the full test-time regression results across all input dimension \(d\) and segment size \(S\).
The advantages of LLA scales with the dimension \(d\) and nonstationarity.
In practical settings, as the dimensionality increases, it is less likely to do exact query \(q=k_j\) as in this synthetic experiment.
Consequently, kernel selectivity becomes less pronounced when noise is present in query points, limiting the potential advantages of both softmax attention and LLA compared to the ideal conditions of this synthetic experiment. Nevertheless, the overall performance trends remain consistent with our main findings.


Figure~\ref{fig:exp_all2} shows complete associative recall results across all sequence lengths \(L\).
For visual clarity, we average results across different numbers of key-value pairs for each sequence length and plot the trajectory of the best score for each model.
Gated DeltaNet exhibits distinctive training dynamics characterized by an extended plateau phase with minimal loss improvement, followed by an abrupt transition to significantly lower test loss and corresponding rapid accuracy improvement.
The timing of this transition is highly sensitive to hyperparameters such as learning rate and dataset size.

In contrast, LLA demonstrates consistent, gradual improvement in test accuracy with a corresponding smooth decrease in test loss throughout training, mirroring the behavior of Softmax Attention but more powerful.
This stable convergence pattern remains robust across a wide range of learning rates and dataset sizes.
The marked difference in optimization dynamics suggests fundamental differences in how these models navigate the loss landscape and converge to solutions.

%% file: arxiv.bbl
\begin{thebibliography}{57}
\expandafter\ifx\csname natexlab\endcsname\relax\def\natexlab#1{#1}\fi
\expandafter\ifx\csname url\endcsname\relax
  \def\url#1{\texttt{#1}}\fi
\expandafter\ifx\csname urlprefix\endcsname\relax\def\urlprefix{}\fi

\bibitem[{Ahn et~al.(2023)Ahn, Cheng, Daneshmand and Sra}]{ahn2023transformerslearnimplementpreconditioned}
\text{Ahn, K.}, \text{Cheng, X.}, \text{Daneshmand, H.} and \text{Sra, S.} (2023).
\newblock Transformers learn to implement preconditioned gradient descent for in-context learning.
\newline\urlprefix\url{https://arxiv.org/abs/2306.00297}

\bibitem[{Aky{\"u}rek et~al.(2022)Aky{\"u}rek, Schuurmans, Andreas, Ma and Zhou}]{akyurek2022learning}
\text{Aky{\"u}rek, E.}, \text{Schuurmans, D.}, \text{Andreas, J.}, \text{Ma, T.} and \text{Zhou, D.} (2022).
\newblock What learning algorithm is in-context learning? investigations with linear models.
\newblock \textit{arXiv preprint arXiv:2211.15661}.

\bibitem[{Arora et~al.(2023)Arora, Eyuboglu, Timalsina, Johnson, Poli, Zou, Rudra and Ré}]{arora2023zoologymeasuringimprovingrecall}
\text{Arora, S.}, \text{Eyuboglu, S.}, \text{Timalsina, A.}, \text{Johnson, I.}, \text{Poli, M.}, \text{Zou, J.}, \text{Rudra, A.} and \text{Ré, C.} (2023).
\newblock Zoology: Measuring and improving recall in efficient language models.
\newline\urlprefix\url{https://arxiv.org/abs/2312.04927}

\bibitem[{Behrouz et~al.(2025)Behrouz, Razaviyayn, Zhong and Mirrokni}]{behrouz2025itsconnectedjourneytesttime}
\text{Behrouz, A.}, \text{Razaviyayn, M.}, \text{Zhong, P.} and \text{Mirrokni, V.} (2025).
\newblock It's all connected: A journey through test-time memorization, attentional bias, retention, and online optimization.
\newline\urlprefix\url{https://arxiv.org/abs/2504.13173}

\bibitem[{Bick et~al.(2025)Bick, Xing and Gu}]{bick2025understandingskillgaprecurrent}
\text{Bick, A.}, \text{Xing, E.} and \text{Gu, A.} (2025).
\newblock Understanding the skill gap in recurrent language models: The role of the gather-and-aggregate mechanism.
\newline\urlprefix\url{https://arxiv.org/abs/2504.18574}

\bibitem[{Bierens(1988)}]{Bierens1988}
\text{Bierens, H.~J.} (1988).
\newblock The nadaraya–watson kernel regression function estimator.
\newblock Working Paper 1988-58, Faculty of Economics and Business Administration, Vrije Universiteit Amsterdam, Amsterdam.

\bibitem[{Dai et~al.(2023)Dai, Sun, Dong, Hao, Ma, Sui and Wei}]{dai2023gptlearnincontextlanguage}
\text{Dai, D.}, \text{Sun, Y.}, \text{Dong, L.}, \text{Hao, Y.}, \text{Ma, S.}, \text{Sui, Z.} and \text{Wei, F.} (2023).
\newblock Why can gpt learn in-context? language models implicitly perform gradient descent as meta-optimizers.
\newline\urlprefix\url{https://arxiv.org/abs/2212.10559}

\bibitem[{Dao(2023)}]{dao2023flashattention2fasterattentionbetter}
\text{Dao, T.} (2023).
\newblock Flashattention-2: Faster attention with better parallelism and work partitioning.
\newline\urlprefix\url{https://arxiv.org/abs/2307.08691}

\bibitem[{Dao et~al.(2022)Dao, Fu, Ermon, Rudra and Ré}]{dao2022flashattentionfastmemoryefficientexact}
\text{Dao, T.}, \text{Fu, D.~Y.}, \text{Ermon, S.}, \text{Rudra, A.} and \text{Ré, C.} (2022).
\newblock Flashattention: Fast and memory-efficient exact attention with io-awareness.
\newline\urlprefix\url{https://arxiv.org/abs/2205.14135}

\bibitem[{Dao and Gu(2024)}]{dao2024transformersssmsgeneralizedmodels}
\text{Dao, T.} and \text{Gu, A.} (2024).
\newblock Transformers are ssms: Generalized models and efficient algorithms through structured state space duality.
\newline\urlprefix\url{https://arxiv.org/abs/2405.21060}

\bibitem[{Dehghani et~al.(2023)Dehghani, Djolonga, Mustafa, Padlewski, Heek, Gilmer, Steiner, Caron, Geirhos, Alabdulmohsin et~al.}]{dehghani2023scaling}
\text{Dehghani, M.}, \text{Djolonga, J.}, \text{Mustafa, B.}, \text{Padlewski, P.}, \text{Heek, J.}, \text{Gilmer, J.}, \text{Steiner, A.~P.}, \text{Caron, M.}, \text{Geirhos, R.}, \text{Alabdulmohsin, I.} \text{et~al.} (2023).
\newblock Scaling vision transformers to 22 billion parameters.
\newblock In \textit{International conference on machine learning}. PMLR.

\bibitem[{Fan et~al.(1996)Fan, Gijbels, Hu and Huang}]{fan1996study}
\text{Fan, J.}, \text{Gijbels, I.}, \text{Hu, T.-C.} and \text{Huang, L.-S.} (1996).
\newblock A study of variable bandwidth selection for local polynomial regression.
\newblock \textit{Statistica Sinica} 113--127.

\bibitem[{Gao et~al.(2024)Gao, Zeng, Du, Cao, Zhou, Qi, Lai, So, Cao, Yang et~al.}]{gao2024seerattention}
\text{Gao, Y.}, \text{Zeng, Z.}, \text{Du, D.}, \text{Cao, S.}, \text{Zhou, P.}, \text{Qi, J.}, \text{Lai, J.}, \text{So, H. K.-H.}, \text{Cao, T.}, \text{Yang, F.} \text{et~al.} (2024).
\newblock Seerattention: Learning intrinsic sparse attention in your llms.
\newblock \textit{arXiv preprint arXiv:2410.13276}.

\bibitem[{Garg et~al.(2023)Garg, Tsipras, Liang and Valiant}]{garg2023transformerslearnincontextcase}
\text{Garg, S.}, \text{Tsipras, D.}, \text{Liang, P.} and \text{Valiant, G.} (2023).
\newblock What can transformers learn in-context? a case study of simple function classes.
\newline\urlprefix\url{https://arxiv.org/abs/2208.01066}

\bibitem[{Gers et~al.(2000)Gers, Schmidhuber and Cummins}]{gers2000learning}
\text{Gers, F.~A.}, \text{Schmidhuber, J.} and \text{Cummins, F.} (2000).
\newblock Learning to forget: Continual prediction with lstm.
\newblock \textit{Neural computation}, \textbf{12} 2451--2471.

\bibitem[{Gu and Dao(2024)}]{gu2024mambalineartimesequencemodeling}
\text{Gu, A.} and \text{Dao, T.} (2024).
\newblock Mamba: Linear-time sequence modeling with selective state spaces.
\newline\urlprefix\url{https://arxiv.org/abs/2312.00752}

\bibitem[{Gu et~al.(2022{\natexlab{a}})Gu, Goel and Ré}]{gu2022efficientlymodelinglongsequences}
\text{Gu, A.}, \text{Goel, K.} and \text{Ré, C.} (2022{\natexlab{a}}).
\newblock Efficiently modeling long sequences with structured state spaces.
\newline\urlprefix\url{https://arxiv.org/abs/2111.00396}

\bibitem[{Gu et~al.(2022{\natexlab{b}})Gu, Johnson, Timalsina, Rudra and Ré}]{gu2022trainhippostatespace}
\text{Gu, A.}, \text{Johnson, I.}, \text{Timalsina, A.}, \text{Rudra, A.} and \text{Ré, C.} (2022{\natexlab{b}}).
\newblock How to train your hippo: State space models with generalized orthogonal basis projections.
\newline\urlprefix\url{https://arxiv.org/abs/2206.12037}

\bibitem[{Hahn(2020)}]{hahn2020theoretical}
\text{Hahn, M.} (2020).
\newblock Theoretical limitations of self-attention in neural sequence models.
\newblock \textit{Transactions of the Association for Computational Linguistics}, \textbf{8} 156--171.

\bibitem[{Hestenes and Stiefel(1952)}]{hestenes1952methods}
\text{Hestenes, M.~R.} and \text{Stiefel, E.} (1952).
\newblock Methods of conjugate gradients for solving linear systems.
\newblock \textit{Journal of Research of the National Bureau of Standards}, \textbf{49} 409--436.

\bibitem[{Jelassi et~al.(2024)Jelassi, Brandfonbrener, Kakade and Malach}]{jelassi2024repeatmetransformersbetter}
\text{Jelassi, S.}, \text{Brandfonbrener, D.}, \text{Kakade, S.~M.} and \text{Malach, E.} (2024).
\newblock Repeat after me: Transformers are better than state space models at copying.
\newline\urlprefix\url{https://arxiv.org/abs/2402.01032}

\bibitem[{Katharopoulos et~al.(2020)Katharopoulos, Vyas, Pappas and Fleuret}]{katharopoulos2020transformersrnnsfastautoregressive}
\text{Katharopoulos, A.}, \text{Vyas, A.}, \text{Pappas, N.} and \text{Fleuret, F.} (2020).
\newblock Transformers are rnns: Fast autoregressive transformers with linear attention.
\newline\urlprefix\url{https://arxiv.org/abs/2006.16236}

\bibitem[{Kirsch et~al.(2024)Kirsch, Harrison, Sohl-Dickstein and Metz}]{kirsch2024generalpurposeincontextlearningmetalearning}
\text{Kirsch, L.}, \text{Harrison, J.}, \text{Sohl-Dickstein, J.} and \text{Metz, L.} (2024).
\newblock General-purpose in-context learning by meta-learning transformers.
\newline\urlprefix\url{https://arxiv.org/abs/2212.04458}

\bibitem[{Lin et~al.(2025)Lin, Nikishin, He and Courville}]{lin2025forgettingtransformersoftmaxattention}
\text{Lin, Z.}, \text{Nikishin, E.}, \text{He, X.~O.} and \text{Courville, A.} (2025).
\newblock Forgetting transformer: Softmax attention with a forget gate.
\newline\urlprefix\url{https://arxiv.org/abs/2503.02130}

\bibitem[{Liu et~al.(2024)Liu, Wang, Wu, Feng, Stone and Liu}]{liu2024longhornstatespacemodels}
\text{Liu, B.}, \text{Wang, R.}, \text{Wu, L.}, \text{Feng, Y.}, \text{Stone, P.} and \text{Liu, Q.} (2024).
\newblock Longhorn: State space models are amortized online learners.
\newline\urlprefix\url{https://arxiv.org/abs/2407.14207}

\bibitem[{Lu et~al.(2025)Lu, Jiang, Liu, Du, Jiang, Hong, Liu, He, Yuan, Wang et~al.}]{lu2025moba}
\text{Lu, E.}, \text{Jiang, Z.}, \text{Liu, J.}, \text{Du, Y.}, \text{Jiang, T.}, \text{Hong, C.}, \text{Liu, S.}, \text{He, W.}, \text{Yuan, E.}, \text{Wang, Y.} \text{et~al.} (2025).
\newblock Moba: Mixture of block attention for long-context llms.
\newblock \textit{arXiv preprint arXiv:2502.13189}.

\bibitem[{Ma et~al.(2024)Ma, Yang, Xiong, Chen, Yu, Zhang, May, Zettlemoyer, Levy and Zhou}]{ma2024megalodonefficientllmpretraining}
\text{Ma, X.}, \text{Yang, X.}, \text{Xiong, W.}, \text{Chen, B.}, \text{Yu, L.}, \text{Zhang, H.}, \text{May, J.}, \text{Zettlemoyer, L.}, \text{Levy, O.} and \text{Zhou, C.} (2024).
\newblock Megalodon: Efficient llm pretraining and inference with unlimited context length.
\newline\urlprefix\url{https://arxiv.org/abs/2404.08801}

\bibitem[{Mahankali et~al.(2023)Mahankali, Hashimoto and Ma}]{mahankali2023stepgradientdescentprovably}
\text{Mahankali, A.}, \text{Hashimoto, T.~B.} and \text{Ma, T.} (2023).
\newblock One step of gradient descent is provably the optimal in-context learner with one layer of linear self-attention.
\newline\urlprefix\url{https://arxiv.org/abs/2307.03576}

\bibitem[{Merrill and Sabharwal(2023)}]{merrill2023parallelism}
\text{Merrill, W.} and \text{Sabharwal, A.} (2023).
\newblock The parallelism tradeoff: Limitations of log-precision transformers.
\newblock \textit{Transactions of the Association for Computational Linguistics}, \textbf{11} 531--545.

\bibitem[{Milakov and Gimelshein(2018)}]{milakov2018onlinenormalizercalculationsoftmax}
\text{Milakov, M.} and \text{Gimelshein, N.} (2018).
\newblock Online normalizer calculation for softmax.
\newline\urlprefix\url{https://arxiv.org/abs/1805.02867}

\bibitem[{Nadaraya(1964)}]{nadaraya1964}
\text{Nadaraya, E.~A.} (1964).
\newblock On estimating regression.
\newblock \textit{Theory of Probability and Its Applications}, \textbf{9} 141--142.

\bibitem[{Peng et~al.(2024)Peng, Goldstein, Anthony, Albalak, Alcaide, Biderman, Cheah, Du, Ferdinan, Hou, Kazienko, GV, Kocoń, Koptyra, Krishna, Jr., Lin, Muennighoff, Obeid, Saito, Song, Tu, Wirawan, Woźniak, Zhang, Zhao, Zhao, Zhou, Zhu and Zhu}]{peng2024eaglefinchrwkvmatrixvalued}
\text{Peng, B.}, \text{Goldstein, D.}, \text{Anthony, Q.}, \text{Albalak, A.}, \text{Alcaide, E.}, \text{Biderman, S.}, \text{Cheah, E.}, \text{Du, X.}, \text{Ferdinan, T.}, \text{Hou, H.}, \text{Kazienko, P.}, \text{GV, K.~K.}, \text{Kocoń, J.}, \text{Koptyra, B.}, \text{Krishna, S.}, \text{Jr., R.~M.}, \text{Lin, J.}, \text{Muennighoff, N.}, \text{Obeid, F.}, \text{Saito, A.}, \text{Song, G.}, \text{Tu, H.}, \text{Wirawan, C.}, \text{Woźniak, S.}, \text{Zhang, R.}, \text{Zhao, B.}, \text{Zhao, Q.}, \text{Zhou, P.}, \text{Zhu, J.} and \text{Zhu, R.-J.} (2024).
\newblock Eagle and finch: Rwkv with matrix-valued states and dynamic recurrence.
\newline\urlprefix\url{https://arxiv.org/abs/2404.05892}

\bibitem[{Peng et~al.(2025)Peng, Zhang, Goldstein, Alcaide, Du, Hou, Lin, Liu, Lu, Merrill, Song, Tan, Utpala, Wilce, Wind, Wu, Wuttke and Zhou-Zheng}]{peng2025rwkv7gooseexpressivedynamic}
\text{Peng, B.}, \text{Zhang, R.}, \text{Goldstein, D.}, \text{Alcaide, E.}, \text{Du, X.}, \text{Hou, H.}, \text{Lin, J.}, \text{Liu, J.}, \text{Lu, J.}, \text{Merrill, W.}, \text{Song, G.}, \text{Tan, K.}, \text{Utpala, S.}, \text{Wilce, N.}, \text{Wind, J.~S.}, \text{Wu, T.}, \text{Wuttke, D.} and \text{Zhou-Zheng, C.} (2025).
\newblock Rwkv-7 "goose" with expressive dynamic state evolution.
\newline\urlprefix\url{https://arxiv.org/abs/2503.14456}

\bibitem[{Poli et~al.(2023)Poli, Massaroli, Nguyen, Fu, Dao, Baccus, Bengio, Ermon and Ré}]{poli2023hyenahierarchylargerconvolutional}
\text{Poli, M.}, \text{Massaroli, S.}, \text{Nguyen, E.}, \text{Fu, D.~Y.}, \text{Dao, T.}, \text{Baccus, S.}, \text{Bengio, Y.}, \text{Ermon, S.} and \text{Ré, C.} (2023).
\newblock Hyena hierarchy: Towards larger convolutional language models.
\newline\urlprefix\url{https://arxiv.org/abs/2302.10866}

\bibitem[{Press et~al.(2022)Press, Smith and Lewis}]{press2022trainshorttestlong}
\text{Press, O.}, \text{Smith, N.~A.} and \text{Lewis, M.} (2022).
\newblock Train short, test long: Attention with linear biases enables input length extrapolation.
\newline\urlprefix\url{https://arxiv.org/abs/2108.12409}

\bibitem[{Ramsauer et~al.(2021)Ramsauer, Schäfl, Lehner, Seidl, Widrich, Adler, Gruber, Holzleitner, Pavlović, Sandve, Greiff, Kreil, Kopp, Klambauer, Brandstetter and Hochreiter}]{ramsauer2021hopfieldnetworksneed}
\text{Ramsauer, H.}, \text{Schäfl, B.}, \text{Lehner, J.}, \text{Seidl, P.}, \text{Widrich, M.}, \text{Adler, T.}, \text{Gruber, L.}, \text{Holzleitner, M.}, \text{Pavlović, M.}, \text{Sandve, G.~K.}, \text{Greiff, V.}, \text{Kreil, D.}, \text{Kopp, M.}, \text{Klambauer, G.}, \text{Brandstetter, J.} and \text{Hochreiter, S.} (2021).
\newblock Hopfield networks is all you need.
\newline\urlprefix\url{https://arxiv.org/abs/2008.02217}

\bibitem[{Schlag et~al.(2021)Schlag, Irie and Schmidhuber}]{schlag2021lineartransformerssecretlyfast}
\text{Schlag, I.}, \text{Irie, K.} and \text{Schmidhuber, J.} (2021).
\newblock Linear transformers are secretly fast weight programmers.
\newline\urlprefix\url{https://arxiv.org/abs/2102.11174}

\bibitem[{Schmidhuber(1992)}]{6796337}
\text{Schmidhuber, J.} (1992).
\newblock Learning to control fast-weight memories: An alternative to dynamic recurrent networks.
\newblock \textit{Neural Computation}, \textbf{4} 131--139.

\bibitem[{Siems et~al.(2025)Siems, Carstensen, Zela, Hutter, Pontil and Grazzi}]{siems2025deltaproductimprovingstatetrackinglinear}
\text{Siems, J.}, \text{Carstensen, T.}, \text{Zela, A.}, \text{Hutter, F.}, \text{Pontil, M.} and \text{Grazzi, R.} (2025).
\newblock Deltaproduct: Improving state-tracking in linear rnns via householder products.
\newline\urlprefix\url{https://arxiv.org/abs/2502.10297}

\bibitem[{Sun et~al.(2023)Sun, Dong, Huang, Ma, Xia, Xue, Wang and Wei}]{sun2023retentivenetworksuccessortransformer}
\text{Sun, Y.}, \text{Dong, L.}, \text{Huang, S.}, \text{Ma, S.}, \text{Xia, Y.}, \text{Xue, J.}, \text{Wang, J.} and \text{Wei, F.} (2023).
\newblock Retentive network: A successor to transformer for large language models.
\newline\urlprefix\url{https://arxiv.org/abs/2307.08621}

\bibitem[{Team(2024)}]{team2024chameleon}
\text{Team, C.} (2024).
\newblock Chameleon: Mixed-modal early-fusion foundation models.
\newblock \textit{arXiv preprint arXiv:2405.09818}.

\bibitem[{Vaswani et~al.(2023)Vaswani, Shazeer, Parmar, Uszkoreit, Jones, Gomez, Kaiser and Polosukhin}]{vaswani2023attentionneed}
\text{Vaswani, A.}, \text{Shazeer, N.}, \text{Parmar, N.}, \text{Uszkoreit, J.}, \text{Jones, L.}, \text{Gomez, A.~N.}, \text{Kaiser, L.} and \text{Polosukhin, I.} (2023).
\newblock Attention is all you need.
\newline\urlprefix\url{https://arxiv.org/abs/1706.03762}

\bibitem[{von Oswald et~al.(2023)von Oswald, Niklasson, Randazzo, Sacramento, Mordvintsev, Zhmoginov and Vladymyrov}]{vonoswald2023transformerslearnincontextgradient}
\text{von Oswald, J.}, \text{Niklasson, E.}, \text{Randazzo, E.}, \text{Sacramento, J.}, \text{Mordvintsev, A.}, \text{Zhmoginov, A.} and \text{Vladymyrov, M.} (2023).
\newblock Transformers learn in-context by gradient descent.
\newline\urlprefix\url{https://arxiv.org/abs/2212.07677}

\bibitem[{von Oswald et~al.(2025)von Oswald, Scherrer, Kobayashi, Versari, Yang, Schlegel, Maile, Schimpf, Sieberling, Meulemans, Saurous, Lajoie, Frenkel, Pascanu, y~Arcas and Sacramento}]{vonoswald2025mesanetsequencemodelinglocally}
\text{von Oswald, J.}, \text{Scherrer, N.}, \text{Kobayashi, S.}, \text{Versari, L.}, \text{Yang, S.}, \text{Schlegel, M.}, \text{Maile, K.}, \text{Schimpf, Y.}, \text{Sieberling, O.}, \text{Meulemans, A.}, \text{Saurous, R.~A.}, \text{Lajoie, G.}, \text{Frenkel, C.}, \text{Pascanu, R.}, \text{y~Arcas, B.~A.} and \text{Sacramento, J.} (2025).
\newblock Mesanet: Sequence modeling by locally optimal test-time training.
\newline\urlprefix\url{https://arxiv.org/abs/2506.05233}

\bibitem[{von Oswald et~al.(2024)von Oswald, Schlegel, Meulemans, Kobayashi, Niklasson, Zucchet, Scherrer, Miller, Sandler, y~Arcas, Vladymyrov, Pascanu and Sacramento}]{vonoswald2024uncoveringmesaoptimizationalgorithmstransformers}
\text{von Oswald, J.}, \text{Schlegel, M.}, \text{Meulemans, A.}, \text{Kobayashi, S.}, \text{Niklasson, E.}, \text{Zucchet, N.}, \text{Scherrer, N.}, \text{Miller, N.}, \text{Sandler, M.}, \text{y~Arcas, B.~A.}, \text{Vladymyrov, M.}, \text{Pascanu, R.} and \text{Sacramento, J.} (2024).
\newblock Uncovering mesa-optimization algorithms in transformers.
\newline\urlprefix\url{https://arxiv.org/abs/2309.05858}

\bibitem[{Wang et~al.(2025)Wang, Shi and Fox}]{wang2025testtimeregressionunifyingframework}
\text{Wang, K.~A.}, \text{Shi, J.} and \text{Fox, E.~B.} (2025).
\newblock Test-time regression: a unifying framework for designing sequence models with associative memory.
\newline\urlprefix\url{https://arxiv.org/abs/2501.12352}

\bibitem[{Watson(1964)}]{watson1964}
\text{Watson, G.~S.} (1964).
\newblock Smooth regression analysis.
\newblock \textit{Sankhy\=a: The Indian Journal of Statistics, Series A}, \textbf{26} 359--372.

\bibitem[{Wortsman et~al.(2023)Wortsman, Liu, Xiao, Everett, Alemi, Adlam, Co-Reyes, Gur, Kumar, Novak et~al.}]{wortsman2023small}
\text{Wortsman, M.}, \text{Liu, P.~J.}, \text{Xiao, L.}, \text{Everett, K.}, \text{Alemi, A.}, \text{Adlam, B.}, \text{Co-Reyes, J.~D.}, \text{Gur, I.}, \text{Kumar, A.}, \text{Novak, R.} \text{et~al.} (2023).
\newblock Small-scale proxies for large-scale transformer training instabilities.
\newblock \textit{arXiv preprint arXiv:2309.14322}.

\bibitem[{Xiao(2025)}]{xiao2025statistics}
\text{Xiao, G.} (2025).
\newblock Statistics behind block sparse attention.
\newblock \url{https://guangxuanx.com/blog/block-sparse-attn-stats.html}.

\bibitem[{Yang et~al.(2025{\natexlab{a}})Yang, Kautz and Hatamizadeh}]{yang2025gateddeltanetworksimproving}
\text{Yang, S.}, \text{Kautz, J.} and \text{Hatamizadeh, A.} (2025{\natexlab{a}}).
\newblock Gated delta networks: Improving mamba2 with delta rule.
\newline\urlprefix\url{https://arxiv.org/abs/2412.06464}

\bibitem[{Yang et~al.(2023)Yang, Wang, Shen, Panda and Kim}]{yang2023gated}
\text{Yang, S.}, \text{Wang, B.}, \text{Shen, Y.}, \text{Panda, R.} and \text{Kim, Y.} (2023).
\newblock Gated linear attention transformers with hardware-efficient training.
\newblock \textit{arXiv preprint arXiv:2312.06635}.

\bibitem[{Yang et~al.(2024)Yang, Wang, Shen, Panda and Kim}]{yang2024gatedlinearattentiontransformers}
\text{Yang, S.}, \text{Wang, B.}, \text{Shen, Y.}, \text{Panda, R.} and \text{Kim, Y.} (2024).
\newblock Gated linear attention transformers with hardware-efficient training.
\newline\urlprefix\url{https://arxiv.org/abs/2312.06635}

\bibitem[{Yang et~al.(2025{\natexlab{b}})Yang, Wang, Zhang, Shen and Kim}]{yang2025parallelizinglineartransformersdelta}
\text{Yang, S.}, \text{Wang, B.}, \text{Zhang, Y.}, \text{Shen, Y.} and \text{Kim, Y.} (2025{\natexlab{b}}).
\newblock Parallelizing linear transformers with the delta rule over sequence length.
\newline\urlprefix\url{https://arxiv.org/abs/2406.06484}

\bibitem[{Ye(2023)}]{ye2023online}
\text{Ye, Z.} (2023).
\newblock From online softmax to flashattention.
\newblock Course notes for "ML for ML Systems" (CSE 599M), University of Washington, Spring 2023.
\newblock Online; University of Washington.
\newline\urlprefix\url{https://courses.cs.washington.edu/courses/cse599m/23sp/notes/flashattn.pdf}

\bibitem[{Yuan et~al.(2025)Yuan, Gao, Dai, Luo, Zhao, Zhang, Xie, Wei, Wang, Xiao et~al.}]{yuan2025native}
\text{Yuan, J.}, \text{Gao, H.}, \text{Dai, D.}, \text{Luo, J.}, \text{Zhao, L.}, \text{Zhang, Z.}, \text{Xie, Z.}, \text{Wei, Y.}, \text{Wang, L.}, \text{Xiao, Z.} \text{et~al.} (2025).
\newblock Native sparse attention: Hardware-aligned and natively trainable sparse attention.
\newblock \textit{arXiv preprint arXiv:2502.11089}.

\bibitem[{Zhang et~al.(2024)Zhang, Frei and Bartlett}]{zhang2024trained}
\text{Zhang, R.}, \text{Frei, S.} and \text{Bartlett, P.~L.} (2024).
\newblock Trained transformers learn linear models in-context.
\newblock \textit{Journal of Machine Learning Research}, \textbf{25} 1--55.

\bibitem[{Zhong et~al.(2025)Zhong, Xu, Ao and Shi}]{zhong2025understandingtransformerperspectiveassociative}
\text{Zhong, S.}, \text{Xu, M.}, \text{Ao, T.} and \text{Shi, G.} (2025).
\newblock Understanding transformer from the perspective of associative memory.
\newline\urlprefix\url{https://arxiv.org/abs/2505.19488}

\end{thebibliography}
